\documentclass[runningheads]{llncs}
\usepackage{amsmath,amsfonts,bm}

\def\1{\bm{1}}

\DeclareMathAlphabet{\mathsfit}{\encodingdefault}{\sfdefault}{m}{sl}
\SetMathAlphabet{\mathsfit}{bold}{\encodingdefault}{\sfdefault}{bx}{n}

\usepackage{graphicx}
\usepackage{graphicx}
\usepackage{times}
\usepackage[noend]{algpseudocode}
\usepackage{soul}
\usepackage{url}
\usepackage[hidelinks]{hyperref}
\usepackage[utf8]{inputenc}
\usepackage[small]{caption}
\usepackage{graphicx}
\usepackage{booktabs}
\usepackage{cite}
\usepackage{algorithm}
\usepackage{float}
\usepackage{multirow}
\usepackage{mathtools}
\usepackage{placeins}
\usepackage{amssymb}
\usepackage{enumitem}
\usepackage{graphicx}
\usepackage{subfig}
\usepackage{mathtools}
\usepackage{wrapfig}
\usepackage{graphicx}
\usepackage{bm}
\usepackage{breqn}
\usepackage{comment}
\usepackage{hyperref}
\usepackage{url}
\usepackage{color}

\newcommand{\x}{\mathbf{x}}
\newcommand{\xs}{\mathbf{x_S}}
\newcommand{\xspi}{\mathbf{x_S^{(i)}}}
\newcommand{\xt}{\mathbf{x_T}}
\newcommand{\xtpi}{\mathbf{x_T^{(i)}}}

\newcommand{\xOne}{\mathbf{x_{{D}_{1}}}}
\newcommand{\xTwo}{\mathbf{x_{{D}_{2}}}}
\newcommand{\xsOne}{\mathbf{x_{{S}_{1}}}}
\newcommand{\xsTwo}{\mathbf{x_{{S}_{2}}}}
\newcommand{\xtOne}{\mathbf{x_{{T}_{1}}}}
\newcommand{\xtTwo}{\mathbf{x_{{T}_{2}}}}

\newcommand{\zOne}{\mathbf{z_{{D}_{1}}}}
\newcommand{\zTwo}{\mathbf{z_{{D}_{2}}}}
\newcommand{\zs}{\mathbf{z_{S}}}
\newcommand{\zsOne}{\mathbf{z_{S_1}}}
\newcommand{\zsTwo}{\mathbf{z_{{S}_{2}}}}
\newcommand{\ztOne}{\mathbf{z_{T_1}}}
\newcommand{\ztTwo}{\mathbf{z_{{T}_{2}}}}
\newcommand{\zsOnepi}{\mathbf{z_{S_1}^{(i)}}}
\newcommand{\ztOnepj}{\mathbf{z_{T_1}^{(j)}}}
\newcommand{\zsTwopi}{\mathbf{z_{S_2}^{(i)}}}

\newcommand{\zTwoH}{\mathbf{\widehat{z}_{D_2}}}
\newcommand{\zsTwoH}{\mathbf{\widehat{z}_{S_2}}}
\newcommand{\ztTwoH}{\mathbf{\widehat{z}_{T_2}}}
\newcommand{\zH}{\mathbf{\widehat{z}}}
\newcommand{\ztH}{\mathbf{\widehat{z}_{T}}}
\newcommand{\zsH}{\mathbf{\widehat{z}_{S}}}
\newcommand{\zdH}{\mathbf{\widehat{z}_D}}
\newcommand{\xT}{\mathbf{\widehat{z}}}

\newcommand{\z}{\mathbf{z}}

\newcommand{\Z}{Z}
\newcommand{\ZHat}{\widehat{Z}}
\newcommand{\gHat}{\widehat{g}}

\newcommand{\Gtwo}{{\bm{\gamma_2}}}
\newcommand{\Gone}{{\bm{\gamma_1}}}

\newcommand{\rpm}{\raisebox{.2ex}{$\scriptstyle\pm$}}
\newcommand{\pt}{p_T}

\newcommand{\ps}{p_S}

\newcommand{\m}{\mathbf{m}}

\makeatletter
\newcommand{\mytag}[2]{%
  \text{#1}%
  \@bsphack
  \begingroup
    \@onelevel@sanitize\@currentlabelname
    \edef\@currentlabelname{%
      \expandafter\strip@period\@currentlabelname\relax.\relax\@@@%
    }%
    \protected@write\@auxout{}{%
      \string\newlabel{#2}{%
        {#1}%
        {\thepage}%
        {\@currentlabelname}%
        {\@currentHref}{}%
      }%
    }%
  \endgroup
  \@esphack
}
\makeatother

\usepackage{amsthm}
\usepackage{mathptmx}
\usepackage{authblk}
\newtheorem{assumption}{Assumption}
\newtheorem*{theorem*}{Theorem}
\newtheorem*{proposition*}{Proposition}
\newtheorem*{lemma*}{Lemma}

\begin{document}

\title{Unsupervised domain adaptation with non-stochastic missing data}
\author{Matthieu Kirchmeyer\inst{1,2} \and Patrick Gallinari\inst{1,2} \and Alain Rakotomamonjy\inst{2,3} \and Amin Mantrach\inst{4}}
\institute{Sorbonne Université, CNRS, LIP6, F-75005 Paris, France \and Criteo AI Lab, Paris, France \and Université de Rouen, LITIS,  France \and Amazon, Luxembourg}
\authorrunning{Matthieu Kirchmeyer et al.} 

\maketitle

\begin{abstract}
We consider unsupervised domain adaptation (UDA) for classification problems in the presence of missing data in the unlabelled target domain. More precisely, motivated by practical applications, we analyze situations where distribution shift exists between domains and where some components are systematically absent on the target domain without available supervision for imputing the missing target components. We propose a generative approach for imputation. Imputation is performed in a domain-invariant latent space and leverages indirect supervision from a complete source domain. We introduce a single model performing joint adaptation, imputation and classification which, under our assumptions, minimizes an upper bound of its target generalization error and performs well under various representative divergence families ($\mathcal{H}$-divergence, Optimal Transport). Moreover, we compare the target error of our Adaptation-imputation framework and the “ideal” target error of a UDA classifier without missing target components. Our model is further improved with self-training, to bring the learned source and target class posterior distributions closer. We perform experiments on three families of datasets of different modalities: a classical digit classification benchmark, the Amazon product reviews dataset both commonly used in UDA and real-world digital advertising datasets. We show the benefits of jointly performing adaptation, classification and imputation on these datasets.
\end{abstract}

\section{Introduction}

Motivated by real applications, we consider a classification problem where: (1) a source and target domain are available with observed source labels and missing target labels, (2) a distribution shift exists between source and target on joint distributions in the input and label space, (3) source input data are fully available while target data have missing input components, which cannot be measured on this domain and (4) there is no possible supervision in the target domain for imputation, thus requiring indirect supervision from the source domain. Furthermore, unobserved features contain complementary information not present in the observed ones so that the former cannot be inferred directly from the latter. (1) and (2) correspond to the classical setting of unsupervised domain adaptation, (3) corresponds to a missing data imputation problem on the target with the difficulty (4). \cite{Rubin1976, Little2002} distinguish three categories of missing data problems based on a missingness mechanism denoted $\phi$. Let $\m$ define a pattern of missing data, $\phi$ defines the conditional distribution $p_{\phi}(\m | \x)$ where $\x$ represents a sample. Missing Completely at Random (MCAR) problems verify $\forall \x, p_{\phi}(\m|\x)=p_{\phi}(\m)$, Missing At Random, $\forall \x, p_{\phi}(\m|\x)=p_{\phi}\left(\m| \x^{\mathrm{obs}}\right)$ with $\x^{\mathrm{obs}}$ the observed feature and Missing Not At Random covers all the other cases. The key idea behind Rubin's theory is that $\m$ is a random variable with a probability distribution and specific imputation approaches were developed for each missingness setting. We consider the setting where target data have systematically missing input components. This corresponds to MCAR with the additional difficulty that $\m$ is deterministic, not stochastic. This problem is more difficult than classical MCAR as neither classical maximum likelihood solutions nor stochasticity in missing features can be used to reconstruct the missing information. While general adaptation and imputation problems were considered independently, there are several instances where they occur simultaneously. This has seldom been analyzed and only for specific cases. We propose a principled solution to this problem under non-stochastic missingness and present practical situations where this occurs. 

There are many problems where specific features in collected data may be systematically absent on a domain. In the literature, this setting is mostly considered when dealing with data with multiple modalities. For example, in disease diagnosis in medical imaging \cite{Cai2018}, for some collected dataset, several modalities are present while they are absent on other datasets for which the corresponding equipment was unavailable. In multi-lingual text classification \cite{Doinychko2020} some collections may be available only for a limited set of languages. Similar considerations hold for recommendation in advertising \cite{Wang2018} and object recognition with multi-sensor data \cite{Tran2017}. The situation which initially motivated our investigation, is the \textit{prospecting} setting in computational advertising. The classical framework for ads on the internet is \textit{retargeting}: users have already interacted with a set of merchant sites and they are targeted when they come back on one of these sites. Retargeting makes use of global user statistics collected on the whole set of merchant sites and of statistics from the specific site the user is browsing. Prospecting aims at targeting a user that visits a site for the first time \cite{Aggarwal2019}; while for such a user, features from his general behavior are available, there is no user information for the targeted site and the corresponding features are absent.
The second issue considered is the distribution shift between domains. For instance, data may be collected on different devices as in medical imaging \cite{Chen2019} or background noise may affect each domain differently. This issue has given rise to the literature of Domain Adaptation when aiming at transferring knowledge from one domain to the other \cite{pan2010}. The ads case described above is subject to both missing data for prospecting users and distribution shift between retargeting and prospecting users as detailed in Section \ref{sec:ads_dataset}.


We propose a model addressing the Adaptation-imputation problem defined by (1) to (4), which learns to perform imputation for the target domain with a conditional generative model. Imputation makes use of indirect supervision from the complete source domain. This allows us to handle non-stochastic missing data, while satisfying the constraints related to adaptation in a latent space and to classification. The imputation process plays an important role, providing us with information about the missing target data while contributing to the alignment and the reconstruction losses. Extensive empirical evidence on handwritten digits, Amazon product reviews and Click-Through-Rate (CTR) prediction domain adaptation problems illustrate the benefit of our model. The original contributions are the following:
\begin{itemize}
    \item We propose a new end-to-end model for handling non-stochastic missing data with domain adaptation. It generates relevant missing information in the latent space conditionally on available information while aligning latent source and target marginals and classifying labelled instances. The joint missing-data and adaptation problem has been seldom considered and never in our context.
    \item We derive an adaptation and an imputation upper bounds. The first one upper bounds our model's target generalization error and is minimized explicitly by our training objective. The second one upper bounds an ideal target error corresponding to an UDA problem without missing features in the target domain.
    \item We improve this model by bringing the source and target class posteriors closer to one another with self-training; this is a useful heuristic when class posteriors mismatch.
    \item We evaluate the model on academic benchmarks and on challenging real-world advertising data and illustrate on these datasets that conditional generative models improve regression-based approaches seen in the literature.
\end{itemize}

\section{Related work}
\label{sec:related_work}
Our problem is related to generic ML topics usually addressed separately e.g. domain adaptation and imputation and in an extend other secondary topics. We provide a brief overview of related contributions in the main topics below and in other minor topics in Appendix \ref{sec:related_work_app}.  

\paragraph{Unsupervised Domain Adaptation}
A number of learning methods approach UDA by weighting individual observations during training \cite{Cortes2014,Lipton2018}. Recent deep learning methods align the source and target distributions by embedding them in a joint latent space. There are two main directions for learning joint embeddings. One is based on adversarial training, making use of GAN extensions; the seminal work of \cite{Ganin2015} learns to map source and target domains onto a common latent space by optimizing jointly 1) an approximation of the $\mathcal{H}$-divergence between the source and target embeddings via adversarial training, 2) a classification term on source data embeddings. This work has been extended in several papers \cite{Tzeng2017,Long2018}. The other direction directly exploits explicit distance measures between source and target representations using Integral Probability Metrics such Maximum Mean Discrepancy \cite{MMDLong} or Wasserstein distance \cite{shen2018,Damodaran2018}. These work consider full input data on both domains. 

\paragraph{Imputation}
Data imputation is addressed by several methods \cite{Little2002,VanBuuren}. Most approaches consider a supervised setting where (1) paired or unpaired complete and incomplete data are available, (2) missingness corresponds to a stochastic process (e.g. a mask distribution for tabular data) and (3) imputation is performed in the original feature space. This is different from our setting when one considers (1) reconstruction in a latent space, (2) imputation for a classification task, (3) no direct supervision and (4) fixed missingness which prevents us from exploiting the statistics from different incomplete samples leading to a much more complex problem. Recently, generative models were adapted for data imputation, e.g. \cite{Yoon2018} and \cite{Mattei2019} for GANs and VAEs respectively. The general approach with generative models is to learn a distribution over imputed data which is similar to the one of plain data. This comes in many different instances and usually, generative training alone is not sufficient; additional loss terms are often used. In paired problems where each missing datum is associated to a plain version, a reconstruction term imposed by a MSE contraint is added \cite{pix2pix2016}; in unpaired problems a cycle-consistency loss is imposed \cite{CycleGAN2017}. \cite{Li2019,Pajot2019} are among the very few approaches addressing unsupervised imputation in which full instances are never directly used. Both extend AmbientGAN \cite{bora2018} and consider stochastic missingness. Our imputation problem is closer to the one addressed in some forms of inpainting \cite{Pathak2016}, missing view imputation \cite{Doinychko2020} or multi-modality missing data \cite{Cai2018}. These approaches are fully supervised. The latter considers, as we do, imputation when one modality is systematically absent, but on one domain only, i.e. without adaptation. \cite{Ding2014,Wei2019,Wei2019_2} are the only papers we are aware of that consider imputation as we do. \cite{Ding2014} considers low-rank constraints and dictionary learning to guide transfer and was not used here as a baseline due to a high complexity that prevents large-scale experiments. \cite{Wei2019,Wei2019_2} are close to our work but assume that missing data can be reconstructed from the observed one through regression. In our setting, this is not possible: given the observed features, there are multiple possible imputations for the missing features; regression is thus meaningless and one has to learn their distribution or at least some modes. This motivates learning a generative model. Moreover, in \cite{Wei2019,Wei2019_2} classification occurs as a downstream task whereas our approach is end-to-end for classification, adaptation and imputation. Finally, our method is theoretically justified and addresses a challenging large size application motivated by a concrete real-world problem never handled before.

\paragraph{Cold-start} Cold-start occurs when making predictions or recommendations when data from the item or user of interest is not available or was not observed in the training set. The standard hypothesis is i.i.d. data coming from the same domain. In recommender systems, several papers address cold-start and leverage auxiliary information about users or items e.g. user attributes, profile, social context or cross-domain information \cite{Barjasteh2015, Sahebi2013}. Cold-start is related to zero-shot learning with unobserved data where usual solutions learn a representation space using auxiliary knowledge e.g. grounded word embeddings with visual context \cite{Zablocki2019}. As for our problem, cold-start deals with non-stochastic missing data, but usually considers only one domain while we deal with distribution shift as well through adaptation.

\section{Problem definition}
\label{sec:objectives}

\paragraph{Notations} 
$\mathcal{X}, \mathcal{Y}$ denote the input and label space. We use $X, Y$ to denote random variables with values in $\mathcal{X}, \mathcal{Y}$. A domain $D$ is defined by a distribution $p_D(X)$ on $\mathcal{X}$ and a deterministic labeling function $f_D:\mathcal{X} \rightarrow \{1, ..., K\}$ where $K$ is the number of classes. $D$ will refer to either the source $S$ or target $T$ domain. Data from domain $D$ is $(\x_D,y_D) \in \mathbb{R}^n \times \{1,...,K\}$ where $n$ is the dimension of the input space, sampled from the domain's joint distribution $p_D(X,Y)$. In the UDA setting, target labels are unknown. We consider that $\x_D$ has two components, $\x_D=(\xOne, \xTwo)$; $X_1, X_2$ refer to each component with values in $\mathcal{X}_1, \mathcal{X}_2$. Given input $\x \in \mathbb{R}^n$, $\mathbf{m} \in \{0,1\}^n$ is a binary mask indicating which entries of $\x$ are missing (1 for missing and 0 for observed). We define $\mathcal{Z}= \mathcal{Z}_1 \times \mathcal{Z}_2$ as the representation space built with a feature extractor. Assuming both components $(\xOne, \xTwo)$ are observed, we define $g$ as
\begin{equation}
    \begin{aligned}
        g:\mathcal{X}_1 \times \mathcal{X}_2  & \rightarrow \mathcal{Z}_1 \times \mathcal{Z}_2 \\
        (\xOne, \xTwo) &\mapsto (g_1(\xOne), g_2(\xTwo))
    \end{aligned}
    \label{eq:g}
\end{equation}
where 
$\zOne=g_1(\xOne)$, $\zTwo=g_2(\xTwo)$ with $\Z_1, \Z_2$, the corresponding random variables.
This is illustrated in Figure \ref{fig:notations} (b) using examples from the \texttt{digits} dataset. While $X_2$ is available on the source domain, it is absent on the target domain. As detailed in Section \ref{sec:training}, we will learn to perform imputation in the latent $\mathcal{Z}$ space via a generative network $r$ operating on $\mathcal{Z}$. For this we will introduce a  
mapping $\gHat$ as follows:
\begin{equation}
    \begin{aligned}
        \gHat:\mathcal{X}_1 &\rightarrow \mathcal{Z}_1 \times \mathcal{Z}_2\\
        \xOne &\mapsto (g_1(\xOne), r \circ g_1(\xOne))
    \end{aligned} 
    \label{eq:gHat}
\end{equation}
where $g_1:\mathcal{X}_1 \rightarrow \mathcal{Z}_1$, $r:\mathcal{Z}_1\rightarrow \mathcal{Z}_2$ and $\zTwoH=r\circ g_1(\xOne)$. $\ZHat_2$ is the corresponding random variable built from $X_1$ with $r \circ g_1$ with values in $\mathcal{Z}_2$ and $\ZHat=(\Z_1, \ZHat_2)$. This is illustrated in Figure \ref{fig:notations} (a). For reasons detailed later, this mapping $r\circ g_1(\cdot)$ will be used on both $S$ and $T$.

\begin{figure}[h!]
    \centering
    \includegraphics[width=0.7\textwidth]{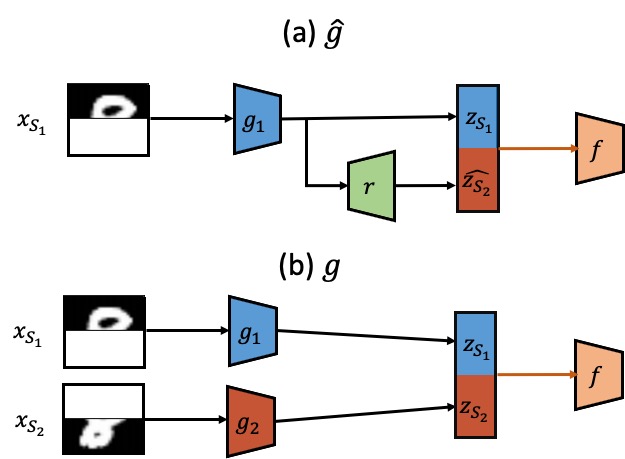}
    \caption{Encoding of an input source digit $\xs=(\xsOne, \xsTwo)$ with $\gHat$ (a) and $g$ (b). $g_1$ encodes the first part of the input $\xsOne$ into $\zsOne$. The second latent component is either built by encoding $\xsTwo$ with $g_2$ as $\zsTwo$ (b) or with reconstruction via $r\circ g_1$ as $\zsTwoH$ (a). These latent components $\zsH$ (a), $\zs$ (b) are fed directly into a classifier $f$.}
    \label{fig:notations}
\end{figure}

\paragraph{Assumptions} 
Let us now introduce formally the different assumptions underlying our context and model.
We address UDA with non-stochastic missing target features and aim at finding a single hypothesis $h_\gHat: \mathcal{X} \rightarrow \{0, ..., K\}$ of the form $f \circ \gHat$, with $\gHat$ the feature extractor defined in \eqref{eq:gHat} and $f$ the classifier with low target risk. Since the problem is under-specified, one has to make assumptions to define it properly:

\begin{assumption}
    Labelled source data $\xs$ are fully observed while unlabelled target data are partially observed with $\xtTwo$ missing. The missingness mechanism corresponds to Missing Completely At Random \cite{Little2002} on the target with fixed missingness pattern. Thus the distribution statistics of the missing data cannot be leveraged for imputation and we can only resort to indirect supervision with adaptation as in Section \ref{sec:training}: we consider only statistics from the source to infer the imputation mechanism as later explained.
    \label{ass:missingness}
\end{assumption}

\begin{assumption}
    The distribution of $X_2|X_1$ projected in the latent space with $g$ from \eqref{eq:g}, $p_D(\Z_2|\Z_1)$, is multi-modal and $Z_1$ and $Z_2$ are not statistically independent. This allows to impute $\zTwo$ given $\zOne$. However, regression on $\zOne$ cannot recover all modes as MSE produces blurry reconstructions by averaging modes. For example, assuming the feature variables $\Z_1, \Z_2$ encode the contour of the top, respectively bottom of a digit, given the bottom contours of a digit, we can reconstruct several candidates of the top contours (the bottom half of 7 can either be reconstructed into 1 or 7; regression will lead to a blurry digit averaging these two modes). As mentioned, this uncertainty is present for all our datasets.
    \label{ass:multi_modal}
\end{assumption}

\begin{assumption}
    The distribution of $X_2|X_1$ projected in the latent space with $g$ in \eqref{eq:g} is the same across domains i.e. $p_S(\Z_2|\Z_1)=p_T(\Z_2|\Z_1)$. This allows us to make use of the source domain information (with available supervision) to infer the target conditional distribution and recover the missing latent component useful for classification. For example, assuming the feature variables $\Z_1, \Z_2$ encode the contour of the top, respectively bottom of a digit, $p(\Z_2|\Z_1)$ is the distribution of the contours of the bottom of the digit given those of the top; it is reasonable to assume that this distribution is the same across domains.
    \label{ass:cond_hyp}
\end{assumption}

\begin{assumption}
    Covariate shift is valid in the latent space obtained with $\gHat$ in \eqref{eq:gHat} i.e. $p_S(Y|\ZHat)=p_T(Y|\ZHat)$ while $p_S(\ZHat) \neq p_T(\ZHat)$. Thus, we can find a classifier $f\circ \gHat$ with low source and target error; this is a common assumption for standard UDA methods. 
    \label{ass:CoS}
\end{assumption}

\section{Adaptation-Imputation model}
As several generative approaches to UDA, we project source and target data onto a common latent space in which data distributions from the two domains should match and learn a classifier using source labels. Our novelty is to offer a solution to deal with datasets with systematically missing data in the target domain. Our model, denoted \textit{Adaptation-Imputation}, is trained to perform three operations jointly: imputation of missing information, alignment of the distributions of both domains and classification of source instances. The three operations are performed in a joint embedding space and all components are trained together with shared parameters. The term imputation is used here in a specific sense: our goal is to recover information from $\xtTwo$ that will be useful for adaptation and for the target data classification objective and not to reconstruct the whole missing $\xtTwo$. This is achieved via a generative model, which for a given datum in $T$ and conditionally on the available information $\xtOne$, attempts to generate the missing information. Because $\xtTwo$ is systematically missing for $T$ (Assumption \ref{ass:missingness}), there is no possible supervision with target samples; instead we use indirect supervision from source samples while transferring to the target. We consider two variants of the same model based on different divergence measures between distributions: the $\mathcal{H}$-divergence approximated through adversarial training (\texttt{ADV}) and the Wasserstein distance (\texttt{OT}) computed through the primal by finding a joint coupling matrix $\gamma$ with linear programming \cite{peyre2019computational}. Our two models can be seen respectively as extensions of DANN \cite{Ganin2015} and DeepJDOT \cite{Damodaran2018} to the missing data problem. We only describe the \texttt{ADV} version in the main text, the extension to \texttt{OT} is detailed in Appendix \ref{sec:ot_formulation}. Results for both models are in Section \ref{sec:experiments}. 

\subsection{Inference}
The latent space representations are denoted $\zdH=(\zOne,\zTwoH)$. $\zOne = g_1(\xOne)$ is the mapping of the observed component $\xOne$ onto the latent space and $\zTwoH = r \circ g_1(\xOne)$ is the second component's latent representation generated conditionally on $\xOne$ through generator $r$, as later described. At inference, given $\xtOne$, we generate $\ztH=(\ztOne,\ztTwoH)$ where $\ztTwoH$ encodes part of the missing information $\xtTwo$ in $\xt$ (Figure \ref{fig:imput_model} (b)). Finally $\ztH$ is fed to the classifier $f$.

\subsection{Training}
\label{sec:training}
For simplicity, we describe each component in turn but please note that they all interact and that their parameters are all optimized according to the three objectives mentioned above. The interaction is discussed after the description of each individual module. The model's components are illustrated in Figure \ref{fig:imput_model} (a).

\paragraph{Adaptation}
Adaptation aligns the distributions of $\zsH$ and $\ztH$ in the latent space.
For $\texttt{ADV}$, alignment is performed via an adversarial loss operating on the latent representations
\begin{equation}
    L_1= \mathbb{E}_{\xs \sim \ps(X)} \log D_1(\zsH)+ \mathbb{E}_{\xt \sim \pt(X)} \log(1-D_1(\ztH))    
    \label{eq:adap_l1}
\end{equation}
where $D_1(\zH)$ represents the probability that $\zH$ comes from $S$ rather than $T$. 

\paragraph{Imputation}
Imputation generates an encoding $\ztTwoH$ for the missing information, conditioned on the available $\xtOne$ thanks to a generative model $r$. Since we never have access to $\xtTwo$, we develop a distant learning strategy: we learn imputation on $S$ through $\zsTwoH=r \circ g_1(\xsOne)$ (Figure \ref{fig:imput_model}) and then transfer to the target domain ($\ztTwoH$ on the figure) via adaptation. For that we perform two operations in parallel. First, we align  the distributions of $\zsTwoH$ and $\zsTwo= g_2(\xsTwo)$ which is the encoding of $\xsTwo$, using an adversarial loss and discriminator $D_2$ ($L_{ADV}$ on Figure \ref{fig:imput_model}). As alignment acts globally on distributions we have no guarantee that $\zsTwoH$ will be associated to the corresponding $\zsOne$. We then  enforce a one-to-one relationship by  associating a $\zsTwoH$ to its specific $\zsOne$. For that, we use a reconstruction term, the MSE distance between $\zsTwo$ and $\zsTwoH$ ($L_{MSE}$ on Figure \ref{fig:imput_model}). This guarantees that the imputed $\zsTwoH$ truly represents information in $\zsTwo$. The learned mappings are used to perform imputation on the target data $\ztTwoH= r \circ g_1(\xtOne)$. The imputation loss $L_2$ has thus two terms: an adversarial term $L_{ADV}$ for aligning $\zsTwo$ and $\zsTwoH$; and a reconstruction term $L_{MSE}$:
\begin{align}
    L_2 &= L_{ADV} + \lambda_{MSE} \times L_{MSE} \label{eq:imput_l2} \\
    L_{ADV} &= \mathbb{E}_{\xsTwo\sim \ps(X_2)} \log D_2(\zsTwoH)+\mathbb{E}_{\xsOne\sim \ps(X_1)} \log(1-D_2(\zsTwo)) \label{eq:imput_adv_l2}\\
    L_{MSE} &= \mathbb{E}_{\xs \sim \ps(X)}\left\|\zsTwo-\zsTwoH\right\|_2^2 \label{eq:imput_mse_l2}
\end{align}
where $\lambda_{MSE}$ weights the regression term over the generative term. Imputation and adaptation influence each other and both are also influenced by classification described below. The latter forces the generated $\zsTwoH$ to contain information about $\xsTwo$ relevant for the classification task. This information is transferred via adaptation to the target when generating $\ztTwoH$.

\paragraph{Classification}
The last component is a classifier $f$, trained on source mappings $\zsH$ as done in classic UDA. The corresponding loss, with $L_{Disc}$ a cross-entropy loss, is 
\begin{equation}
    L_3=\mathbb{E}_{(\xs, y_S) \sim \ps(X,Y)}L_{Disc}(f(\zsH), y_S)
    \label{eq:classif_l3}
\end{equation}

\paragraph{Overall loss}
$L$ is the weighted sum of the adaptation, imputation and classification losses  
\begin{equation}
    L = \lambda_1 \times L_1 + \lambda_2 \times L_2 + \lambda_3 \times L_3
    \label{eq:overall_loss}
\end{equation}
with $\lambda_1, \lambda_2, \lambda_3$ some hyperparameters and we solve
\begin{equation}
    \min_{g_{1},g_{2}, r, f} ~\max_{D_1,D_2} L
    \label{eq:optim_pb}
\end{equation}

\begin{figure}[h!]
    \centering
    \includegraphics[width=0.8\textwidth]{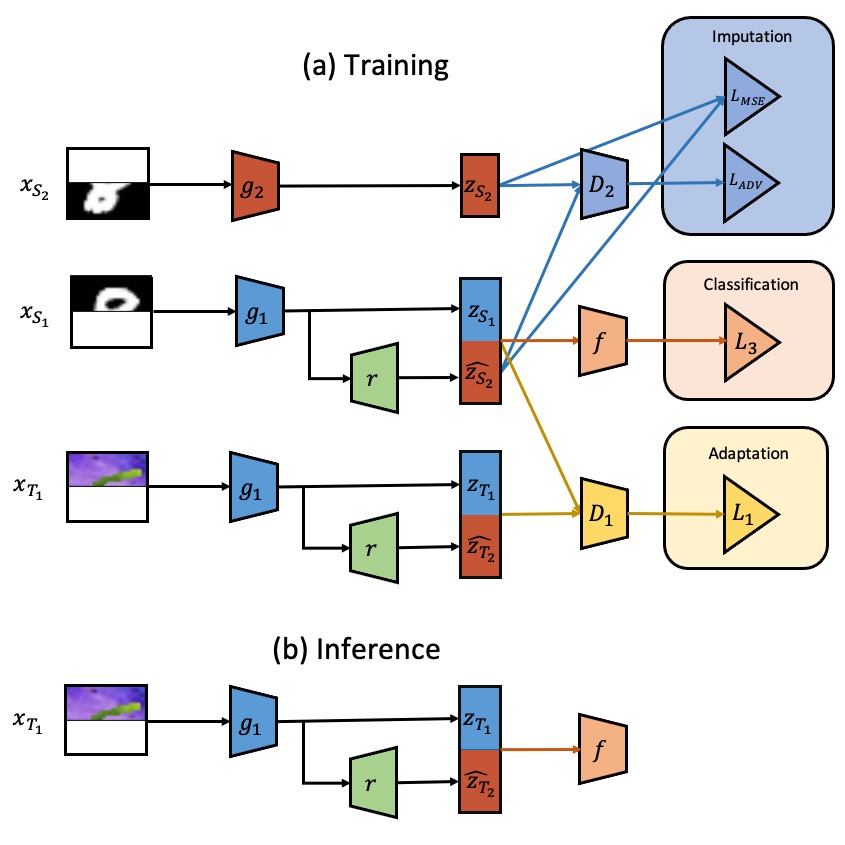}
    \caption{\textit{Adaptation-Imputation} model. The first column represents examples of raw data with missing and non-missing parts. Trapezoidal boxes represent mapping functions. Triangles in the last column represent loss functions used only for training. At training, the top-row depicts how $\xsTwo$ is mapped into the latent space with $g_2$. The second and third rows show how $\zsH$ and $\ztH$ are obtained. All these imputed and mapped source and target samples are then used in training losses. At inference, we only need the learned $g_1$ and $r$ for mapping the target example with missing data into the latent space and $f$ for predicting its class.}
    \label{fig:imput_model}
\end{figure}

\paragraph{Interaction between the model's components} 
Mappings $g_1, g_2, r$ appear in the three terms of $L$, meaning that they should learn to perform the three tasks simultaneously. $g_1$ maps $\xsOne$ and $\xtOne$ onto the latent space, the embeddings being denoted respectively $\zsOne$ and $\ztOne$. $r$ learns to generate missing information $\zTwoH$ from $\zOne$. $\zdH$ is generated to fulfill the classification objective.
$g_2$ should fulfill the imputation objective while preserving part of the information present in $\xsTwo$. Our model uses a unique mapping $g_1$ for both $S$ and $T$; compared to using separate mappings, this reduces the number of parameters and was found to perform as well.

\paragraph{Implementation}
For adversarial training, discriminators $D_1$ (adaptation) and $D_2$ (imputation) are implemented by binary classifiers. $D_1$ is trained to distinguish $\zsH$ from $\ztH$ mappings while $D_2$ is trained to separate imputed $\zsTwoH$, generated from $\xsOne$, and $\zsTwo$, a direct embedding of $\xsTwo$. We use gradient reversal layers \cite{Ganin2015} for implementing the min-max condition on $D_1$ and $D_2$. To stabilize adversarial training, we update progressively $\lambda_1, \lambda_2$, respectively the hyperparameter for the adaptation loss $L_1$ and the imputation loss $L_2$, from 0 to 1 when updating the feature extractors $g_1, g_2$. Both $\lambda_1$ and $\lambda_2$ are set to 1 when updating the discriminators $D_1, D_2$ per \cite{Ganin2015}. Moreover, we decay all learning rates. We fix $\lambda_3=1$ to avoid additional tuning and only tune $\lambda_{MSE}$ as shown in the ablation study in Table \ref{table:imputation_ablation} and Figure \ref{fig:ablation_digits}. All components are trained jointly after first initializing the classifier $f$ and feature extractors $g_1, g_2$ to minimize $L_3$ replacing $\zsTwoH$ with $\zsTwo$ such that discriminitative components are learned before joint adaptation and imputation. Appendix \ref{sec:implementation_details} provides details of all architectures and parameters and our code is available\footnote{\url{https://github.com/mkirchmeyer/adaptation-imputation}}.

\begin{algorithm}[h!]
    \caption{Adversarial Adaptation-Imputation training procedure}
    \label{alg:pseudo-code}
        \hspace*{\algorithmicindent} 
        $N$: number of epochs, $k$: batch size
        \begin{algorithmic}[1]
            \State Initialize $f, g_1, g_2$ by minimizing $L_3$ replacing $\zsTwoH$ with $\zsTwo$
            \For{$n_{epoch} < N$}
                \State Sample $\{\mathbf{x_S^{(i)}}, y_S^{(i)}\}_{1 \leq i \leq k}$ from $\ps(X,Y)$
                \State Sample $\{\mathbf{x_{T}^{(j)}}\}_{1 \leq j \leq k}$ from $\pt(X)$ 
                \State Decay learning rate and update gradient scale at each batch
                \State Compute $L = \lambda_1 L_1 + \lambda_2 L_2 + \lambda_3 L_3$ performing joint adaptation, imputation, classification 
                \State Update $D_1$, $D_2$ by ascending $L$ through Gradient Reversal Layer 
                \State Update $f, g_1, g_2, h$ by descending $L$
            \EndFor   
        \end{algorithmic}
\end{algorithm}

\section{Theoretical insights}

\subsection{Target generalization error}
Given the model in Section \ref{sec:training}, we show in this section that, despite having only unlabelled target samples, we minimize the model's target classification error using source labels with an adaptation upper bound (Theorem \ref{theorem:adaptation_bound}), under our assumptions. We then show an imputation upper bound of the "ideal" target error obtained with all components observed (classical UDA setting) by our model's target error times a factor (Proposition \ref{prop:imputation_bound}). The analytical expression of this factor highlights the role of two components: imputation on the source and transfer of this imputation from the source to the target. In the optimal case, when the model perfectly recovers these two operations, we show that our model retrieves the ideal target error. These two bounds thus provide an approach with adaptation and imputation to minimize our model's target error and reach the ideal target error, using only missing target data and source supervision for both labels and imputation. 

\paragraph{Definitions}
First, we recall some definitions. $\gHat$ in \eqref{eq:gHat} maps the first component of a sample to its imputed latent representation. $\gHat$ can be applied to both source and target samples. On the other hand, $g$ in \eqref{eq:g} maps both input components on the latent space and is thus only applicable to source samples. In practise, $\gHat$ and $g$ share the same encoder for the first component $g_1$; the second encoder is respectively $r \circ g_1$ for $\widehat{g}$ and $g_2$ for $g$. The random variables associated to these projections are denoted respectively $\Z_2$, for the latent missing component built from $X_2$ with $g_2$ and $\ZHat_2$, for the reconstruction of $\Z_2$ from $X_1$ with $r \circ g_1$. $X_2$ is missing on $T$ but observed on $S$. Based on these mappings, we define the risk of a hypothesis $h$ on domain $D\in\{S,T\}$, either $h_g \in \mathcal{H}_{g}=\{f \circ g: f \in \mathcal{F}\}$ or $h_{\gHat} \in \mathcal{H}_{\gHat}=\{f \circ \gHat: f \in \mathcal{F}\}$, as its error under the true labeling function $f_D$ and $D$, i.e. 
$$\epsilon_D(h) \triangleq \epsilon_D(h, f_D) \triangleq \mathbb{E}_{\x \sim p_D(X)}[|h(\x)-f_D(\x)|]$$ In our case, as $h$ and $f_D$ are binary classification functions, this definition reduces to the probability that $h$ disagrees with $f_D$ under $p_{D}(X)$ 
$$\epsilon_D(h)=\mathbb{E}_{\x \sim p_D(X)}[|h(\x)-f_D(\x)|]=\mathbb{E}_{\x \sim p_D(X)}[\mathbb{I}(h(\x) \neq f_D(\x))]=\operatorname{Pr}_{\x \sim p_D(X)}(h(\x) \neq f_D(\x))$$

In the following, we describe the adaptation and imputation bounds. 

\paragraph{Adaptation bound}
As target samples are unlabelled, we cannot directly minimize our model's target error, $\epsilon_T(f\circ \gHat)$. In practise, we upper bound $\epsilon_T(f\circ \gHat)$ in Theorem \ref{theorem:adaptation_bound} with adaptation. Adaptation is performed on both components despite target missingness thanks to imputation which reconstructs the missing latent component conditionally on the observed one.
\begin{theorem}[Proof in Appendix  \ref{sec:proof}]
    \label{theorem:adaptation_bound}
   Given $f \in \mathcal{F}, \gHat$ in \eqref{eq:gHat} and $\ps(\ZHat)$, $\pt(\ZHat)$ the latent marginal distributions obtained with $\gHat$.
    \begin{equation}
        \begin{aligned}
            \epsilon_T(f\circ \gHat) \leq \underbrace{\left[\epsilon_S(f\circ \gHat)+d_{\mathcal{F} \Delta \mathcal{F}}(\ps(\ZHat), \pt(\ZHat))+\lambda_{\mathcal{H}_{\gHat}}\right]}_{\mathrm{Domain~Adaptation~} \mytag{(DA)}{term_DA}}
        \end{aligned}
        \label{eq:uda_bound}
    \end{equation}
    with $\epsilon_{S}(\cdot), \epsilon_{T}(\cdot)$ the expected error under the labelling function $f_S, f_T$ on $S$, $T$ respectively; $\mathcal{F}\Delta\mathcal{F}$ the symmetric difference hypothesis space\footnote{$h \in \mathcal{F}\Delta\mathcal{F} \iff h(\x)=f_1(\x)\oplus f_2(\x)$ for some $f_1, f_2 \in \mathcal{F}$ where $\oplus$ is the XOR function.}; $d_{\mathcal{H}}$ the $\mathcal{H}$-divergence for $\mathcal{H}=\mathcal{F}\Delta\mathcal{F}$ and $\lambda_{\mathcal{H}_{\gHat}}=\min_{f' \in \mathcal{F}}\left[\epsilon_{S}(f' \circ \gHat)+\epsilon_{T}(f'\circ \gHat)\right]$, the joint risk of the optimal hypothesis.
\end{theorem}
The upper bound in \eqref{eq:uda_bound} consists of $\epsilon_S(f\circ \gHat)$ assessing the discriminative information of source latent components and $d_{\mathcal{F} \Delta \mathcal{F}}(\ps(\ZHat), \pt(\ZHat))+\lambda_{\mathcal{H}_{\gHat}}$, assessing the transfer to the target. Our model minimizes this upper bound \ref{term_DA}; $L_3$ in  \eqref{eq:classif_l3} corresponds to the first term while $L_1$ in  \eqref{eq:adap_l1} to the second. Assumption \ref{ass:CoS} allows us to consider the third term as small. Adaptation affects both components $(\Z_1, \ZHat_2)$ as the missing component is imputed with $r \circ g_1$, yet, imputation here is not supervised with fully observed components. 

\paragraph{Imputation bound}
Given $f \in \mathcal{F}$, we compare under our assumptions $\epsilon_T(f\circ \gHat)$ and the ideal target error with full data, $\epsilon_T(f\circ g)$, with $g=(g_1, g_2)$ and $\gHat=(g_1, r\circ g_1)$. This allows us to measure the loss in performance due to missingness when using $f\circ \gHat$ instead of $f\circ g$. $g_1$ is shared in $g$ and $\gHat$ while $r\circ g_1$ reconstructs the missing component on both domains. We first derive Lemma \ref{lemma:lemma_s_bound} used in our upper bound in Proposition \ref{prop:imputation_bound}. 
\begin{lemma}[Proof in Appendix  \ref{sec:proof}]
    For any continuous density distributions $p$, $q$ defined on an input space $\mathcal{X}$, such that $\forall \x \in \mathcal{X}, q(\x) > 0$, the inequality $\sup_{\x \in \mathcal{X}}[p(\x)/q(\x)] \geq 1$ holds. Moreover, the minimum is reached when $p=q$.
    \label{lemma:lemma_s_bound}
\end{lemma}

We derive Proposition \ref{prop:imputation_bound}. Under Assumption \ref{ass:cond_hyp}, given a classifier $f \in \mathcal{F}$ and encoders $g, \gHat$, this proposition upper bounds $\epsilon_T(f\circ g)$ with $\epsilon_T(f\circ \gHat)$ multiplied by a factor \ref{term_It} in \eqref{eq:imputation_bound}. Our model minimizes both the Adaptation upper bound and the term \ref{term_It}. 

\begin{proposition}[Proof in Appendix  \ref{sec:proof}]
    Under Assumption \ref{ass:cond_hyp}, let $f \in \mathcal{F}, \gHat$ \eqref{eq:gHat} and $g$ \eqref{eq:g},
    \begin{equation}
        \epsilon_T(f\circ g) \leq \underbrace{\underbrace{\sup_{\z\sim p(\Z)}[\dfrac{p_S(\Z_2=\z_2|\z_1)}{p_S(\ZHat_2=\z_2|\z_1)}]}_{\mathrm{Imputation~error~on~ S~} \mytag{(IS)}{term_Is}} \times \underbrace{\sup_{\z\sim p(\Z)}[\dfrac{p_S(\ZHat_2=\z_2|\z_1)}{p_T(\ZHat_2=\z_2|\z_1)}]}_{\mathrm{Transfer~error~of~Imputation~} \mytag{(TI)}{term_Ti}}}_{\mathrm{Imputation~error~on~T~} \mytag{(IT)}{term_It}} \times \epsilon_T(f\circ \gHat)
        \label{eq:imputation_bound}
    \end{equation}
    Under Lemma \ref{lemma:lemma_s_bound}, \ref{term_It}=1 is the minimal value reached when $p_S(\Z_2|\Z_1)=p_S(\ZHat_2|\Z_1)$ and $p_S(\ZHat_2|\Z_1)=p_T(\ZHat_2|\Z_1)$. In this case, $\epsilon_T(f\circ g) = \epsilon_T(f\circ \gHat)$.
    \label{prop:imputation_bound}
\end{proposition}
The upper bound in \eqref{eq:imputation_bound} shows that for any $f, \gHat, g$, $\epsilon_T(f\circ g)$ is upper bounded by $\epsilon_T(f\circ \gHat)$ times the multiplicative factor \ref{term_It}. The optimal situation, equality, is obtained when \ref{term_It} equals 1. \ref{term_It} measures how imputation recovers the missing target component and is decomposed into two terms. \ref{term_Is} quantifies how imputation learns $p_S(\Z_2|\Z_1)$ with $p_S(\ZHat_2|\Z_1)$ i.e. reconstructs the component $Z_2=g_2(X_2)$ with $\ZHat_2=r(Z_1)$ and $Z_1=g_1(X_1)$ on the source. \ref{term_Ti} measures the divergence of $\ZHat_2|\Z_1$ across domains; the lower, the better indirect imputation supervision from $S$ transfers to $T$. The equality case occurs when \ref{term_It} is minimal, i.e when $p_S(\Z_2|\Z_1)=p_S(\ZHat_2|\Z_1)$ and $p_S(\ZHat_2|\Z_1)=p_T(\ZHat_2|\Z_1)$. Our model minimizes \ref{term_It} after first initializing $f, g$ with arg$\min_{f,g}\epsilon_S(f\circ g)$ replacing $\gHat$ with $g$ in $L_3$, \eqref{eq:classif_l3} to extract discriminative components $(\zsOne, \zsTwo)$. It minimizes \ref{term_Is} with $L_2$ in \eqref{eq:imput_l2} while \ref{term_Ti} is minimized with the adaptation loss $L_1$ in \eqref{eq:adap_l1}. Note that \ref{term_It} is minimal when $L_1=L_2=0$ yielding to the  equality of $\epsilon_T(f\circ g)$ and $\epsilon_T(f\circ \gHat)$. 

\subsection{Self-training refinement $\mathcal{R}$} 
We now introduce a heuristic based on pseudo-labels useful for settings where Assumption \ref{ass:CoS} is not verified because $p_S(Y|\ZHat) \neq p_T(Y|\ZHat)$. Assumption \ref{ass:CoS} allows to consider $\lambda_{\mathcal{H}_{\gHat}}$ in \eqref{eq:uda_bound} as small. Indeed, several authors e.g. \cite{Zhao2019,Johansson2019} recently demonstrated that minimizing the first two terms in \ref{term_DA} \eqref{eq:uda_bound} is not sufficient for successful UDA. They show that (1) even when covariate shift is true in the data space, it usually does not hold in the latent space; (2) even when the first two terms in \ref{term_DA} in \eqref{eq:uda_bound} are minimized, the third, $\lambda_{\mathcal{H}_{\gHat}}$, might increase so that the bound is not minimized. \cite{Zhao2019} shows that in addition to the above conditions, one should enforce the posterior class distributions $p_D(Y|X)$ to be close on the two domains. Since $T$ is unlabeled there is no direct way to do that. We instead propose a simple heuristic using pseudo-labels and show how they can be incorporated with a simple adaptation of \eqref{eq:uda_bound}. Pseudo-labels are tentative labels assigned to target unlabelled samples by a classifier, denoted $h_{\gHat}$ below. 
As $\lambda_{\mathcal{H}_{\gHat}}$ cannot be measured without target labels, we will approximately evaluate and minimize it with pseudo-labels. 

\begin{proposition}[Proof in Appendix  \ref{sec:proof}]
    Assume a joint distribution $p_{\widetilde{T}}(X,Y)$ where $p_{\widetilde{T}}(X) = p_T(X)$ and $Y = h_\gHat(X)$ where $h_\gHat = f \circ \gHat \in \mathcal{H}_\gHat$ is a candidate hypothesis. Then,
    \begin{equation}
        \lambda_{\mathcal{H}_{\gHat}} \leq \min_{h_\gHat \in \mathcal{H}_\gHat} \left[\epsilon_S(h_\gHat) + \epsilon_{\widetilde{T}}(h_\gHat) + \epsilon_{T}(f_{\widetilde{T}})\right]
        \label{eq:extended_uda_bound}
    \end{equation}
    with $\epsilon_{T}(f_{\widetilde{T}})=\operatorname{Pr}_{\x \sim p_T(X)}(f_{\widetilde{T}}(\x) \neq f_T(\x))$ the $T$ error of the pseudo-labelling function $f_{\widetilde{T}}$.
    \label{prop:self_training}
\end{proposition}

The first two terms on the right hand side of \eqref{eq:extended_uda_bound} may be controlled as we know source labels and target pseudo-labels; the third term is the error of the pseudo-labeling function, minimal if pseudo-labels are equal to true target labels. We cannot measure the last term but propose self-training as a way to heuristically improve the pseudo-labeling function. 

We detail one way to do so in Algorithm \ref{alg:refinement-pseudo-code}. We start from an initial set of pseudo-labels, e.g. the pseudo-labels provided by the model in Section \ref{sec:training} and then refine them. Many self-training methods have been proposed. We use a combination of two such methods, initially proposed for semi-supervised learning: an adaptation of the semi-supervised discriminant Classification Expectation Maximization (CEM) in \cite{Amini2005} and semi-supervised learning by entropy minimization \cite{Grandvalet2004}. We found that combining these two approaches performed better than each method used alone.

In the following we assume to have a set $\mathcal{S}, \mathcal{T}$ respectively of labelled $S$ and unlabelled $T$ samples. \cite{Amini2005} introduce an iterative method which starts from pseudo-labels provided by an initial classifier and retrains the classifier with these labels. We start with $f(\xT)$ trained as in Section \ref{sec:training} and keep, at each iteration, all samples in $\mathcal{T}$ whose classification score is above a threshold, this set of pseudo-labelled instances is denoted $\mathcal{T}^{pl}$. We then minimize a cross-entropy loss on $\mathcal{S} \cup \mathcal{T}^{pl}$, between the labels for $\mathcal{S}$ or pseudo-labels for $\mathcal{T}^{pl}$ and the predicted scores. \cite{Grandvalet2004} optimizes an entropy loss on the distribution of the predicted class posteriors output from $f$ for all unlabelled samples; we apply this loss to $\mathcal{T} \setminus \mathcal{T}^{pl}$. This entropy loss can be considered as a soft version of the discriminant CEM loss. 

In conclusion, we first train the model without pseudo-labels minimizing $L$ (Section \ref{sec:training}). We then use the learned classifier to provide initial pseudo-labels and minimize jointly discriminant CEM and entropy loss to refine them. Given $h_\gHat=f\circ \gHat\in \mathcal{H}_\gHat$ a hypothesis with $\forall k \in [1, K], {h_\gHat}_k(\x)$ the probability of predicting instance $\x$ to class $k$, $L_{Disc}$ a cross-entropy loss and $\lambda$ a weight for entropy, the objective function of our refinement method is:
\begin{equation}
    L_{\mathcal{R}} = \underbrace{\sum_{(\x, y) \in \mathcal{S} \cup \mathcal{T}^{pl}} L_{Disc}(h_\gHat(\x), y)}_{\text{Discriminant~CEM}~ \mytag{(CEM)}{term_CEM}} + \lambda \underbrace{\sum_{\x \in \mathcal{T} \setminus \mathcal{T}^{pl}} \sum_{k=1}^{K} {h_\gHat}_k(\x) \log {h_\gHat}_k(\x)}_{\mathrm{Entropy~} \mytag{(E)}{term_E}}
    \label{eq:refinement_eq}
\end{equation}
The first term in \eqref{eq:refinement_eq}, \ref{term_CEM}, controls $\epsilon_S(h_\gHat) + \epsilon_{\widetilde{T}}(h_\gHat)$ while the second term, \ref{term_E}, heuristically controls $\epsilon_{T}(f_{\widetilde{T}})$ by encouraging separation between classes. We found that this heuristically brings pseudo-labels closer to the target labels on our datasets. In practise, we minimize $L_{\mathcal{R}}$ with respect to $f, \gHat$.

\begin{algorithm}[h!]
    \caption{Self-training procedure for Adaptation-Imputation}
    \label{alg:refinement-pseudo-code}
        \textbf{Input} $\mathcal{S}=\{(\xspi, {y_S^{(i)}})\}_{i=1}^{N_S}$, $\mathcal{T}=\{(\xtpi\}_{i=1}^{N_T}$, Adaptation-Imputation method $\mathcal{A}$ in Section \ref{sec:training} \\
        \textbf{Output} Classifier $f$; Feature extractor $\gHat$ defined in \eqref{eq:gHat}
        \hspace*{\algorithmicindent} 
        \begin{algorithmic}[1]
            \item $f, \gHat = \mathcal{A}(\mathcal{S}, \mathcal{T})$ \Comment{Initialize $f, \gHat$ with \textit{Adaptation-Imputation} \eqref{eq:overall_loss}} 
            \State $f, \gHat = \text{arg}\min_{f,\gHat} L_{\mathcal{R}}$ \Comment{Semi-supervised refinement of $f, \gHat$ by optimizing \eqref{eq:refinement_eq}}
        \end{algorithmic}
\end{algorithm}

\section{Experiments}
\label{sec:experiments}

\subsection{Datasets and experimental setting}

\paragraph{Datasets}
Experiments are performed on three types of datasets. The first one, $\texttt{digits}$, is a classical multi-class classification benchmark used in many UDA studies and adapted to fit our missing data setting. The second one, which initially motivated our framework, consists of advertising datasets where we aim at transferring knowledge from retargeting users with full browsing information to prospecting users with missing information. The task is binary classification as measured by Click-Through-Rate (CTR) or Conversion Rate (CR)\footnote{{\scriptsize CTR is the number of clicks made on ads divided by the number of shown ads. CR replaces clicks with purchases.}} given user browsing traces. We use two such datasets: \texttt{ads-kaggle} is a public kaggle dataset\footnote{{\scriptsize http://labs.criteo.com/2014/02/kaggle-display-advertising-challenge-dataset/}}, while \texttt{ads-real} was gathered internally. Both correspond to real advertising traffic. Finally, we performed tests on a text dataset, Amazon reviews, denoted \texttt{amazon}. The initial problem is transformed into binary classification and to a non-stochastic missing data problem. For both \texttt{digits} and \texttt{amazon}, a subset of the components are set to $0$ to mimic missing data while on \texttt{ads}, data is missing structurally (more details in Appendix \ref{sec:dataset_description}).

\paragraph{Baselines}
We report results for the following models: 

(a) \textit{Source-Full} trained without adaptation on $\xs$ and tested on full $\xt$; adaptation is added in \textit{Adaptation-Full}. Note that this model is only applicable for our academic benchmark where we have access to full data.

(b) \textit{Source-ZeroImputation} and \textit{Adaptation-ZeroImputation} do the same but considering full $\xs$ while $\xt$ is incomplete. Missing data $\xtTwo$ is set to $\textbf{0}$, $\xt=(\xtOne,\textbf{0})$.

(c) \textit{Source-IgnoreComponent} and \textit{Adaptation-IgnoreComponent} are a variant of the above where only $\xOne$ is considered while $\xTwo$ is ignored for both $S$ and $T$. 

(d) \textit{Adaptation-Imputation}, our model, considers full $\xs$ and $\xt=(\xtOne,\textbf{0})$ adding imputation with a conditional generative model. 

(e) We add self-training to \textit{Adaptation-Imputation} and when applicable to \textit{Adaptation-Full}.

Note that \textit{Adaptation-Full} is an upper bound of our imputation model since it uses full information while $\xtTwo$ is not available in practice. \textit{Adaptation-ZeroImputation} and \textit{Adaptation-IgnoreComponent} are lower bounds for our model since they only perform adaptation and do not impute non-zero values. 

\paragraph{Hyperparameters} Parameters are chosen using the DEV estimator \cite{You2019}. For \texttt{digits}, NN architectures are adapted from \cite{Ganin2015}; we use Adam optimizer with $lr=10^{-2}$ decayed; batch size of $128$ and $100$ epochs. For \texttt{ads} and \texttt{amazon}, three-layered NN with $128$ neurons per layer are used as feature extractors; the classifier and discriminators are single-layered with $128$ neurons; $lr=10^{-6}$ and is decayed; batch size is $500$ with $50$ epochs. Reported results are mean value and standard deviation over five runs and best results are indicated in $\textbf{bold}$. Further details are given in the Appendix \ref{sec:hyperparameter}.

\subsection{Digits}
\label{sec:result_digits}
\paragraph{Description}
We consider UDA problems between several datasets: MNIST \cite{MNIST}, USPS \cite{Hull:1994:DHT:628312.628607}, SVHN \cite{SVHN} and MNIST-M \cite{Ganin2015} as illustrated in Figure  \ref{fig:missing_digits} (a). MNIST $\rightarrow$ SVHN is not considered as it is difficult for traditional UDA \cite{Ganin2015}. All tasks are 10-class classification problems. From complete digits datasets, we build datasets with missing input values by setting corresponding pixel values to zero for horizontal patches of different sizes as illustrated on Figure  \ref{fig:missing_digits} (b) for MNIST-M digits. It is clear that there is domain shift on these datasets as the pixel values have different mean and variance across domains.

\begin{figure}[h!]
    \centering
    \includegraphics[width=0.8\textwidth]{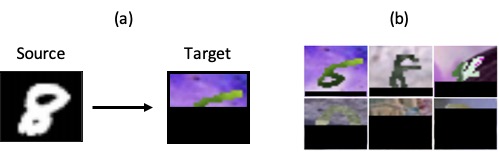}
    \caption{(a) MNIST$\rightarrow$MNIST-M adaptation; (b) Digits with missing horizontal patches of increasing size}
    \label{fig:missing_digits}
\end{figure}
\paragraph{Results with half of the digit missing}
We first removed half of each target digit, the horizontal bottom part. We report target accuracy in Table \ref{table:imputation_results} for both \texttt{ADV} and \texttt{OT} models. Removing half of the digit leads to a strong performance decrease for \textit{Source-IgnoreComponent} and \textit{Source-ZeroImputation} compared to the upper-bounds of \textit{Source-Full}; the performance is partially recovered with adaptation. \textit{Adaptation-Imputation} clearly improves on \textit{Adaptation-IgnoreComponent} and \textit{Adaptation-ZeroImputation} in all cases which validates the importance of imputation. 
However, it does not reach the  upper bound performance of \textit{Adaptation-Full}. 
Both \texttt{ADV} and \texttt{OT} versions exhibit the same behavior. In the results in Table \ref{table:imputation_results}, \texttt{ADV} performance is higher than \texttt{OT}. This is because performance is highly dependent on the NN architectures and we tuned our NNs for \texttt{ADV}. \texttt{OT} models may reach performance similar to \texttt{ADV} but require an order of magnitude more parameters. To keep the comparison fair, we use the same NN models for both \texttt{ADV} and \texttt{OT}. Imputation models achieve their highest performance when adaptation between domains is complex (MNIST $\rightarrow$ MNIST-M, SVHN $\rightarrow$ MNIST) illustrating the importance of imputation when transfer is difficult. We show in Appendix \ref{sec:appendix_embedding} the learned latent representations $\zsH, \ztH$ for various \texttt{digits} adaptation problems.

\paragraph{Varying missing patch size}
We analyze the impact of the size of the missing patch by removing a percentage $p \in \{30\%, 40\%, 50\%, 60\%,70\%\}$ of MNIST digits when adapting SVHN $\rightarrow$ MNIST, with the same hyperparameters. Mean values over five runs are reported in Figure \ref{fig:missing_patch} for \texttt{ADV} models. We notice that our model constantly beats the other baselines regardless of the missing patch size. The figure exhibits borderline cases when the size of the missing patch becomes very small ($<30\%$) or very large ($>65\%$). When the missing patch is small there is enough information for predicting the label thus simple models perform well; when it becomes big, there is not enough information for efficient reconstructions.

\begin{figure}[H]
    \centering
    \includegraphics[width=0.7\textwidth]{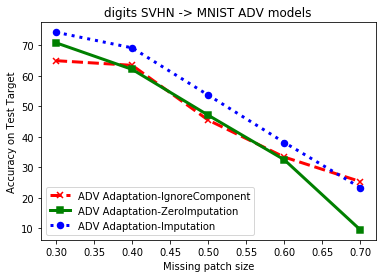}
    \caption{\texttt{ADV} target accuracy ($\uparrow$) on SVHN $\rightarrow$ MNIST with missing patch size}
    \label{fig:missing_patch}
\end{figure}

\subsection{Ads}
\label{sec:ads_dataset}

\paragraph{Description}
The \texttt{ads} datasets are used for solving the binary classification problem of predicting if a \textbf{user} exposed to an ad from a \textbf{partner} (e.g. Booking.com) clicks given his browsing history. A row in this dataset is a vector $\x=(\xOne, \xTwo)$ specific to a (user-partner) pair where $\xOne$ gathers mean statistics for this user on all visited partners summarizing the user's display and click statistics and $\xTwo$ corresponds to the user-partner specific traces. The label is the response to an ad for this (user, partner) pair, a click for \texttt{ads-kaggle} or a purchase for \texttt{ads-real}. We transfer knowledge from the labelled source domain composed of all user-partner pairs for which the user has already interacted with the partner (\textbf{retargeting} users) to the unlabelled target domain composed of all the user-partner pairs for which the user has never interacted with this partner (\textbf{prospecting} users). $\xsTwo$ is known but $\xtTwo$ is \textbf{unknown}. There are several partners and users per domain. These datasets are large scale as seen in Table \ref{table:stats_criteo} (1M and 24M source displays respectively for \texttt{ads-kaggle}, \texttt{ads-real}) with some specificities: there is class imbalance and five times less data on the target than on the source. For both datasets besides missingness, there is also an adaptation problem: prospecting users tend to be less active and their statistics are usually different from those of retargeting users, with a higher overall activity (e.g. in terms of frequency of a partner's website visits); this translates into distribution shifts on $\xOne$ across domains. We visualize in the Appendix in Table \ref{table:hist_tab_kaggle} and Figure \ref{fig:hist_kaggle} the domain shift in \texttt{ads-kaggle} which comprises 13 features. Table \ref{table:hist_tab_kaggle} reports mean and standard deviation on each feature's value over a domain and Figure \ref{fig:hist_kaggle} plots the histogram of the distribution of each feature where the $y$-axis is unnormalized and corresponds to real counts. Feature 5 is naturally missing on $T$ and distributions are different in shape, mean and variance across domains. To show the benefit of modelling additional missing features, we artificially set features 1, 6, 7, 11 and 12 to zero on $T$ such that in total 6 features are missing while 7 are present. On \texttt{ads-real}, 12 features are missing while 17 are present and we observe the same domain shift trend; however missing features are naturally missing and we do not have access to their value. 

\paragraph{Results}
We report results in Table \ref{table:imputation_results} only for \texttt{ADV} models as we observed that the trend is similar for both \texttt{ADV} and \texttt{OT}. Missing features are structurally missing in the datasets, so we cannot report results for models using full inputs. The classes being imbalanced, accuracy is not relevant here so we report the log cross-entropy (CE) between the predicted values and the true labels. CE is considered to be the most reliable metric to estimate revenue for the ads problem and for large user bases small CE improvements can lead to a large revenue increase. For \texttt{ads-kaggle}, an improvement of 0.001 in CE is considered as significant \cite{DNC}. A first observation is that the imputation model is substantially better than the baselines on both datasets. For \texttt{ads-kaggle} it improves by $2.3\%$ the best adaptation model i.e. the adaptation model with zero imputation while for \texttt{ads-real} the improvement reaches $6.3\%$ over the second-best \textit{Source-IgnoreComponent}. A second observation is that for any model, adaptation consistently improves over the model without adaptation. The only exception is the setting ignoring the missing component in \texttt{ads-real}. A third observation is that there is a benefit of imputing the missing component for classification: source CE (not reported) shows that \textit{Source-ZeroImputation} which exploits $\xTwo$ is consistently higher than \textit{Source-IgnoreComponent} which does not, leading to relative gains of 5.6$\%$ on \texttt{ads-kaggle} and 8.2$\%$ on \texttt{ads-real}. The imputation model is able to generate and exploit this information.

\begin{table*}[ht]
    \centering
    \resizebox{\linewidth}{!}{
        {\begin{tabular}{ccccccccccc} 
            \hline
            Dataset & \multicolumn{2}{c}{MNIST $\rightarrow$ USPS} & \multicolumn{2}{c}{USPS $\rightarrow$ MNIST} & \multicolumn{2}{c}{SVHN $\rightarrow$ MNIST} & \multicolumn{2}{c}{MNIST $\rightarrow$ MNIST-M} &
            \multicolumn{1}{c}{$\texttt{ads-kaggle}$} & \multicolumn{1}{c}{$\texttt{ads-real}$} \\ 
            \hline
            \hline
            
            Model w/o $\mathcal{R}$ & \texttt{ADV} & \texttt{OT} & \texttt{ADV} & \texttt{OT} & \texttt{ADV} & \texttt{OT} & \texttt{ADV} & \texttt{OT} & \multicolumn{1}{c}{\texttt{ADV}} & \multicolumn{1}{c}{\texttt{ADV}} \\ \hline
            \hline
            Source-Full & \multicolumn{2}{c}{71.5\rpm2.7} & \multicolumn{2}{c}{74.2\rpm2.7} & \multicolumn{2}{c}{58.1\rpm1.1} & \multicolumn{2}{c}{28.3\rpm1.4} & \multicolumn{2}{c}{NA} \\
            Adaptation-Full & 85.8\rpm3.2 & 92.6\rpm1.7 & 94.6\rpm2.1 & 93.9\rpm0.6 & 78.0\rpm3.4 & 76.1\rpm1.4 & 60.8\rpm3.8 & 46.9\rpm3.9 & \multicolumn{2}{c}{NA} \\
            \hline
            \hline
            Source-ZeroImputation & \multicolumn{2}{c}{25.7\rpm3.7} & \multicolumn{2}{c}{39.2\rpm2.6} & \multicolumn{2}{c}{31.5\rpm2.} & \multicolumn{2}{c}{14.4\rpm1.1} & 0.545\rpm0.019 & 0.663\rpm0.011 \\
            Adaptation-ZeroImputation & 48.4\rpm4.8 & 60.9\rpm6.3 & 67.5\rpm2.2 & 65.3\rpm5.2 & 47.1\rpm5.7 & 37.5\rpm6.2 & 34.7\rpm2.5 & 20.2\rpm2.5 & 0.397\rpm0.0057 & 0.660\rpm0.025 \\
            \hline
            Source-IgnoreComponent & \multicolumn{2}{c}{52.9\rpm9.7} & \multicolumn{2}{c}{54.3\rpm1.6} & \multicolumn{2}{c}{44.6\rpm1.9} & \multicolumn{2}{c}{19.1\rpm2.6} & 0.406\rpm0.00046 & 0.622\rpm0.0048 \\
            Adaptation-IgnoreComponent & 71.5\rpm3.2 & 64.0\rpm5.0 & 80.0\rpm1.4 & 72.0\rpm1.8 & 45.5\rpm1.9 & 47.9\rpm1.8 & 29.4\rpm1.6 & 26.8\rpm4.4 & 0.403\rpm0.0030 & 0.634\rpm0.0082 \\
            \hline
            Adaptation-Imputation & \textbf{74.2\rpm2.3} & \textbf{66.8\rpm1.3} & \textbf{81.4\rpm0.8} & \textbf{72.5\rpm2.7} & \textbf{53.8\rpm1.4} & \textbf{49.2\rpm1.5} & \textbf{57.9\rpm2.3} & \textbf{29.2\rpm1.4} & \textbf{0.389\rpm0.014} & \textbf{0.583\rpm0.013} \\
            \hline
        \end{tabular}}
    }
    \caption{Best target accuracy ($\uparrow$) on \texttt{digits} and CE ($\downarrow$) on \texttt{ads} without $\mathcal{R}$}
    \label{table:imputation_results}
\end{table*}

\subsection{Amazon reviews}

\paragraph{Description}
Besides dealing with images and interaction features in the \texttt{digits} and \texttt{ads} datasets, we also performed experiments on an additional modality, text. \texttt{amazon} is the Amazon product review dataset \cite{Blitzer2006} with four domains (Books, DVDs, Electronics, and Kitchen) transformed to binary classification with positives referring to reviews with rating above 3 stars and negatives to reviews with rating below 3 stars. Additional details on data processing can be found in Appendix \ref{sec:dataset_description}. We consider four adaptation problems and simulate missing features by setting the first half of the features to zero.

\paragraph{Results}
Results are reported in Table \ref{table:amazon_results} and confirm our prior findings i.e. that jointly performing adaptation and imputation improves our baselines. We also notice that our model achieves similar performance to models using full data showing that imputation successfully recovered the missing component.

\begin{table*}[ht]
    \centering
    \scalebox{0.7}{
            \begin{tabular}{ccccc}
                \hline
                Dataset & \multicolumn{1}{c}{DVD $\rightarrow$ Electronics} & \multicolumn{1}{c}{Books $\rightarrow$ Kitchen} & \multicolumn{1}{c}{Kitchen $\rightarrow$ Electronics} & \multicolumn{1}{c}{DVD $\rightarrow$ Books} \\ 
                \hline
                \hline
                Source-Full & $69.57$ & $73.04$ & $77.88$ & $71.95$ \\
                Adaptation-Full & $73.62$ & $74.09$ & $79.63$ & $72.65$ \\
                \hline
                \hline
                Source-ZeroImputation & $58.51$ & $60.52$ & $66.27$ & $61.15$ \\
                Adaptation-ZeroImputation & $64.51$ & $61.08$ & $68.02$ & $62.80$ \\
                \hline
                Source-IgnoreComponent & $60.21$ & $62.03$ & $67.62$ & $64.35$ \\
                Adaptation-IgnoreComponent & $61.02$ & $64.08$ & $68.47$ & $66.00$ \\
                \hline
                Adaptation-Imputation & $\textbf{72.57}$ & $\textbf{72.69}$ & $\textbf{78.18}$ & $\textbf{72.61}$ \\
                \hline
            \end{tabular}
    }
    \caption{Best target accuracy ($\uparrow$) on \texttt{amazon} without $\mathcal{R}$}
    \label{table:amazon_results}
\end{table*}

\subsection{Refinement $\mathcal{R}$}
Results with pseudo-labels are reported in Table \ref{table:imputation_refinement} on \texttt{digits} and \texttt{ads-kaggle} for \textit{Adaptation-Full} and \textit{Adaptation-Imputation}. We set the threshold score selection for the discriminative CEM component to $95\%$ i.e. the pseudo labels of all target instances $\xt$ s.t. $\max_k {h_\gHat}_k(\xt) \geq 0.95$ are considered to be true and set the entropy weight to $\lambda=0.1$ on \texttt{digits} and $\lambda=1$ on \texttt{ads-kaggle}. Learning rates used for solving \eqref{eq:optim_pb} are divided by $10$ and $10$ epochs of successive refinement steps are applied. We observe a clear global improvement on both datasets showing that our refinement model is a good heuristic on real-world datasets for which we usually have $p_S(Y|\ZHat) \neq p_T(Y|\ZHat)$. For standard UDA methods such as \textit{Adaptation-Full}, performance is significantly improved everywhere with small change on MNIST $\rightarrow$ MNIST-M; \textit{Adaptation-Full} is not measurable for \texttt{ads-kaggle}. Our imputation with refinement model follows the same trend with a considerable relative gain of $+18.5\%$ on \texttt{ads-kaggle}.

\begin{table*}[ht]
    \centering
    \resizebox{\linewidth}{!}{
    {\begin{tabular}{cccccc}
        \hline
        \texttt{ADV} Model & MNIST $\rightarrow$ USPS & USPS $\rightarrow$ MNIST & SVHN $\rightarrow$ MNIST & MNIST $\rightarrow$ MNIST-M & $\texttt{ads-kaggle}$ \\
        \hline
        \hline
        Adaptation-Full w/ $\mathcal{R}$ & \textbf{95.9\rpm0.6 (+12\%)} & \textbf{96.8\rpm0.6 (+2.3\%)} & \textbf{83.3\rpm3.9 (+6.8\%)} & \textbf{60.9\rpm3.7 (+0.2\%)} & NA \\
        Adaptation-Imputation w/ $\mathcal{R}$ & \textbf{78.5\rpm1.6 (+5.8\%)} & \textbf{82.5\rpm0.5 (+1.4\%)} & \textbf{58.6\rpm1.8 (+8.9\%)} & \textbf{58.2\rpm2.3 (+0.5\%)} & \textbf{0.317\rpm0.0023 (+18.5\%)} \\
        \hline
    \end{tabular}}}
    \caption{$\mathcal{R}$ with relative gain over Table \ref{table:imputation_results}; target accuracy ($\uparrow$) on \texttt{digits} and CE ($\downarrow$) on \texttt{ads}}
    \label{table:imputation_refinement}
\end{table*}

\subsection{Ablation analysis}
\label{sec:ablation_studies}
We analyze the importance of each component of our model on the public datasets (\texttt{digits}, \texttt{amazon} and \texttt{ads-kaggle}) and report results in Table \ref{table:imputation_ablation} (bottom) and Figure \ref{fig:ablation_digits}.
\paragraph{Adaptation}
We measure the effect of adaptation term $L_1$ \eqref{eq:adap_l1} in $L$ in Table \ref{table:imputation_ablation} (first row). When removing adaptation, inference is performed as before by feeding $\ztH$ to $f$. This means that we only rely on the imputation and classification losses to learn the parameters of the model. For all datasets, adding $L_1$ considerably increases performance.
\paragraph{Imputation}
Imputation $\zsTwoH=h \circ g_1(\xsOne)$, combines adversarial training (\texttt{ADV}) and conditioning on the input datum via MSE (\texttt{MSE}) in $L_2$ \eqref{eq:imput_l2}.
\texttt{ADV} aligns the distributions of $\zsTwo$ and $\zsTwoH$ while \texttt{MSE} can be thought as performing regression. For a given $\xsOne$, there are possibly several potential $\xsTwo$ and thus $\zsTwo$. \texttt{ADV} allows us to focus on a specific mode of $\zsTwo$, while \texttt{MSE} will favour a mean value of the distribution. Results in Table \ref{table:imputation_ablation} (second row), show that for our datasets, combining \texttt{MSE} and \texttt{ADV} leads to improved results compared to using separately each loss. \texttt{MSE} alone already provides good performance, while using only \texttt{ADV} is clearly uncompetitive. Note that reconstruction is an ill-posed problem since the task is inherently ambiguous (different digits may be reconstructed from a half image). We performed tests with a stochastic input component to recover different modes, but the performance was broadly similar. We investigate in Figure \ref{fig:ablation_digits} several weighted combinations of \texttt{MSE} and \texttt{ADV}: for \texttt{digits} and \texttt{amazon}, equal weights were found to be a good choice, while for \texttt{ads-kaggle} performance is improved with other weightings. On Figure \ref{fig:ablation_digits}, \texttt{ADV} induces a high variance in the results (left part of $x$-axis) while \texttt{MSE} stabilizes the performance (right part of $x$-axis). \texttt{ADV} allows for better performance at the expense of high variance; a small contribution from \texttt{MSE}, $\lambda_{MSE} = 0.005$, stabilizes the results.

\begin{table*}[h!]
    \centering
    \resizebox{\linewidth}{!}{
    {\begin{tabular}{ccccccc}
        \hline
        Ablation study & \texttt{ADV} Model & MNIST $\rightarrow$ USPS & USPS $\rightarrow$ MNIST & SVHN $\rightarrow$ MNIST & MNIST $\rightarrow$ MNIST-M & $\texttt{ads-kaggle}$ \\
        \hline
        \hline
        \multirow{1}{*}{$L_{2}+L_{3}$ vs. $L_1 + L_{2}+L_{3}$} & $L=\lambda_2 L_{2}+ \lambda_3 L_{3}$ & 64.2\rpm1.8 (-13\%) & 51.3\rpm2.5 (-37\%) & 44.5\rpm1.4 (-17\%) & 24.1\rpm2.6 (-58\%) & 0.410\rpm0.0020 (-5.4\%) \\
        \hline
        \multirow{4}{*}{\texttt{ADV}-\texttt{MSE} weighting in $L_2$} & $L_2=L_{MSE}$ & 71.9\rpm3.7 (-3.1\%) & \textbf{81.4\rpm1.2 (0\%)} & 52.5\rpm3.7 (-2.4\%) & 56.5\rpm2.8 (-2.4\%) & 0.400\rpm0.0014 (-2.8\%)\\
        & $L_2=L_{ADV}$ & 28.6\rpm3.2 (-61\%) & 39.4\rpm5.2 (-52\%) & 28.8\rpm3.8 (-46\%) & 30.0\rpm3.7 (-48\%) & 0.469\rpm0.13 (-21\%)\\
        & $L_2=L_{ADV}+0.005 \times L_{MSE}$ & 47.8\rpm3.7 (-36\%) & 49.6\rpm5.8 (-39\%) & 46.0\rpm2.6 (-15\%) & 50.6\rpm2.2 (-13\%) & \textbf{0.389\rpm0.014 (0\%)}\\
        & $L_2=L_{ADV}+L_{MSE}$ & \textbf{74.2\rpm2.3 (0\%)} & \textbf{81.4\rpm0.8 (0\%)} & \textbf{53.8\rpm1.4 (0\%)} & \textbf{57.9\rpm2.3 (0\%)} & 0.401\rpm0.0014 (-3.1\%) \\
        \hline
        \hline
        Ablation study & \texttt{ADV} Model & \multicolumn{1}{c}{DVD $\rightarrow$ Electronics} & \multicolumn{1}{c}{Books $\rightarrow$ Kitchen} & \multicolumn{1}{c}{Kitchen $\rightarrow$ Electronics} & \multicolumn{1}{c}{DVD $\rightarrow$ Books} & \multicolumn{1}{c}{} \\ 
        \cline{1-6}
        \multirow{2}{*}{\texttt{ADV}-\texttt{MSE} weighting in $L_2$} & $L_2=L_{MSE}$ & 71.47 (-1.5\%) & 71.39 (-1.8\%) & 77.58 (-0.77\%) & 72.02 (-0.81\%) & \multicolumn{1}{c}{}\\
        & $L_2=L_{ADV}+L_{MSE}$ & $\textbf{72.57 (0\%)}$ & $\textbf{72.69 (0\%)}$ & $\textbf{78.18 (0\%)}$ & $\textbf{72.61 (0\%)}$ & \multicolumn{1}{c}{}\\
        \cline{1-6}
    \end{tabular}}}
    \caption{Ablation with relative gain over Table \ref{table:imputation_results}; accuracy ($\uparrow$) on \texttt{digits}, \texttt{amazon} and CE ($\downarrow$) on \texttt{ads}}
    \label{table:imputation_ablation}
\end{table*}

\begin{figure}[]
    \centering
    \includegraphics[width=0.7\textwidth]{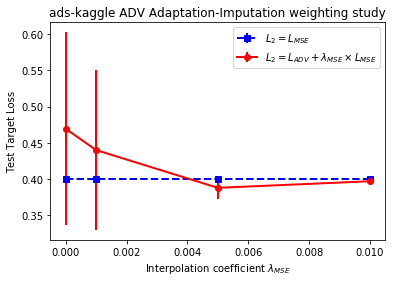}
    \caption{\textit{Adaptation-Imputation} target CE ($\downarrow$) with standard deviations on \texttt{ads-kaggle} w.r.t. $\lambda_{MSE}$}
    \label{fig:ablation_digits}
\end{figure}

\subsection{Discussion}
\paragraph{Relationship between theoretical and experimental results} We comment on our experimental results in light of our adaptation \eqref{eq:uda_bound} and imputation \eqref{eq:imputation_bound} upper-bounds. Let us first consider \eqref{eq:uda_bound}. The first term in \eqref{eq:uda_bound}, $\epsilon_S(f\circ \gHat)$, is the classification loss $L_3$ in \eqref{eq:classif_l3}. The second term in \eqref{eq:uda_bound} $d_{\mathcal{F} \Delta \mathcal{F}}(\ps(\ZHat), \pt(\ZHat))$ is approximated by a  proxy $L_1$ \eqref{eq:adap_l1} and accounts for alignment. $L_1$ leads to substantial gains in Table \ref{table:imputation_ablation} (first row) when added to the loss. The third term $\lambda_{\mathcal{H}_{\gHat}}$ in \eqref{eq:uda_bound} is the optimal joint error heuristically controlled with self-training as justified by upper-bound \eqref{eq:extended_uda_bound}, with gains shown in Table \ref{table:imputation_refinement}. Second, we consider \eqref{eq:imputation_bound}. It is the product of two terms, the target imputation error $(I_T)$ and the error on the target $\epsilon_T(f\circ \gHat)$ which is exactly the left hand side term in bound \eqref{eq:uda_bound}. $(I_T)=(I_S)\times (T_I)$, $(I_S)$ is the source imputation error and is optimized when term $L_2$ \eqref{eq:imput_l2} is zero. $(T_I)$ is the transfer error, optimized when $L_1$ \eqref{eq:adap_l1} is zero. Adding $L_1$ to the loss improves the performance (Table \ref{table:imputation_ablation}). $L_2$ \eqref{eq:imput_l2} explains the gains of \textit{Adaptation-Imputation} over \textit{Adaptation-ZeroImputation} in Table \ref{table:imputation_results} as \textit{Adaptation-ZeroImputation} does not attempt to impute missing components. To summarize, minimizing our global error function $L$ in \eqref{eq:overall_loss} minimizes, according to the approximations just described, the two upper bounds in \eqref{eq:uda_bound} and \eqref{eq:imputation_bound}.

\paragraph{Limitations}
Our results are obtained under some assumptions which we are the first to introduce to our knowledge for our problem. First, if the missing and the observed components are statistically independent, Assumption \ref{ass:multi_modal} is not valid, and then there is no way to impute this missing data. Second, if $p_S(\Z_2|\Z_1) \neq p_T(\Z_2|\Z_1)$ i.e. Assumption \ref{ass:cond_hyp} is not valid, then we cannot transfer imputation from source to target. Yet, these assumptions are most often met in applications and allow to build a well-defined model with good empirical results. 

\section{Conclusion}
We proposed a new model for UDA with non-stochastic target missingness with indirect supervision from a complete source. This method uses only labelled source instances imputing the missing target values in a latent space. Under our assumptions, it minimizes an adaptation upper-bound of its target error and an imputation upper-bound of the ideal target error with full data and leads to important gains for two representative families of divergences (\texttt{OT}, \texttt{ADV}) on our benchmarks (digits, amazon) and on real-world advertising datasets, which are a complex task with missing features. We show that approaches using a pure regressive generator underperform compared to our approach on our real-world applications for which distributions are multi-modal. Finally, we introduced a heuristic refinement method based on self-training to deal with settings where posterior distributions mismatch. As follow-up, we plan to further investigate how to generate diverse outputs in our imputation network.

\paragraph{Acknowledgements}
We would like to thank Keerthi Selvaraj for useful discussions. Alain Rakotomamonjy is funded by RAIMO ANR-20- CHIA-0021-01 and OATMIL ANR-17-CE23-0012 Projects of the French National Research Agency (ANR).

\bibliographystyle{splncs04}
\bibliography{imputation}   

\newpage
\appendix
\section{Additional related work}
\label{sec:related_work_app}

We present in this section some other secondary topics related to our problem in complement to Section \ref{sec:related_work}. 
\paragraph{Concept drift in data streams} Adapting to non i.i.d. data is also considered in evolving data streams where concept drift may occur \cite{Gama2014}. The hypotheses are different from the ones in our setting where adaptation is performed between static domains. 
\paragraph{Batch effect and multiple environments} Data may come from different environments with different distributions. Classical learning frameworks like ERM consider shuffled data without making the distinction between environments which may lead to erroneous conclusions. In biology, this is known as the batch effect \cite{Leek2010}. In ML, recent papers learn domain invariant representations from different environments \cite{Arjovsky2019}. This is different from the situation considered here where one explicitly adapts from a source to a target environment.

\section{\texttt{OT} Adaptation-Imputation formulation}
\label{sec:ot_formulation}
We present here in more details our model using Optimal Transport (OT) as a divergence metric. The formulation is slightly different compared to $\texttt{ADV}$ models. We replace the $\mathcal{H}$-divergence approximation given by the discriminators $D_1$ and $D_2$ by the Wasserstein distance between source and target instances ($D_1$) and true and imputed feature representations ($D_2$), following the original ideas in \cite{shen2018, Damodaran2018}. 
In practice, we compute the Wasserstein distance using its primal form by finding a joint coupling matrix $\gamma$, using a linear programming approach \cite{peyre2019computational}. In \cite{Damodaran2018,jdot}, the OT problem is formulated on the joint $p(X, Y)$ distributions. Similarly to \cite{shen2018}, in our case, we focus on a plan that acts only on the feature space without taking care of the labels. This leads to:
\begin{dmath}
    L_{1} = \sum_{ij} \left(||\zsOnepi-\ztOnepj||^2 + ||{\zsTwoH}^{(i)}-{\ztTwoH}^{(j)}||^{2} \right) \Gone_{ij}
    \label{eq:djdot1}
\end{dmath}
where $\Gone_{ij}$ is the alignment value between source instance $i$ and target instance $j$.

For the imputation part, we keep the reconstruction MSE component in Equation \ref{eq:imput_mse_l2} and derive the distribution matching loss as:
\begin{dmath}
    L_{OT} = \sum_{ij} ||\zsTwopi - \zsTwoH^{(j)}||^2 \Gtwo_{ij}
    \label{eq:djdot_ot}
\end{dmath}
where $\Gtwo_{ij}$ is the alignment value between source instance $i$ and $j$. The final imputation loss is:
\begin{equation}
    L_2 = \lambda_{OT} \times L_{OT} + \lambda_{MSE} \times L_{MSE}
\label{eq:imput_ot}
\end{equation}

The classification term in Equation \ref{eq:classif_l3} is unchanged.

\medskip
The optimization problem in Equation \ref{eq:optim_pb} is solved in two stages following an alternate optimization strategy:
\begin{itemize}
    \item We fix all parameters but $\Gone$ and $\Gtwo$ and find the joint coupling matrices $\Gone$ and $\Gtwo$ using EMD $\min_{\Gone, \Gtwo} L$
    \item We fix $\Gone$ and $\Gtwo$ and solve $\min_{g_1, g_2, r, f} L$
\end{itemize}

In practice, we first minimize $L_3$ for a couple of epochs (taken to be 10 for \texttt{digits}) then minimize $\lambda_1 L_1 + \lambda_2 L_2 + \lambda_3 L_3$ in the remaining epochs. Learning rate and parameters are detailed further in Section \ref{sec:implementation_details}.

\section{Proofs}
\label{sec:proof}

\begin{theorem*}[\textbf{\ref{theorem:adaptation_bound}}]
   Given $f \in \mathcal{F}, \gHat$ in \eqref{eq:gHat} and $\ps(\ZHat)$, $\pt(\ZHat)$ the latent marginal distributions obtained with $g$.
    \begin{equation*}
        \begin{aligned}
            \epsilon_T(f\circ \gHat) \leq \underbrace{\left[\epsilon_S(f\circ \gHat)+d_{\mathcal{F} \Delta \mathcal{F}}(\ps(\ZHat), \pt(\ZHat))+\lambda_{\mathcal{H}_{\gHat}}\right]}_{\mathrm{Domain~Adaptation~} \ref{term_DA}}
        \end{aligned}
    \end{equation*}
    with $\epsilon_{S}(\cdot), \epsilon_{T}(\cdot)$ the expected error w.r.t to the labelling function $f_S, f_T$ on $S$, $T$ respectively; $\mathcal{F}\Delta\mathcal{F}$ the symmetric difference hypothesis space; $d_{\mathcal{H}}$ the $\mathcal{H}$-divergence for $\mathcal{H}=\mathcal{F}\Delta\mathcal{F}$ and \\$\lambda_{\mathcal{H}_{\gHat}}=\min_{f' \in \mathcal{F}}\left[\epsilon_{S}(f' \circ \gHat)+\epsilon_{T}(f'\circ \gHat)\right]$, the joint risk of the optimal hypothesis.
\end{theorem*}
\begin{proof}
    We apply \cite{Ben-David2010} to form the bound in $\mathcal{Z}$ using $\gHat$. 
\end{proof}

\begin{lemma*}[\textbf{\ref{lemma:lemma_s_bound}}]
    For any continuous density distribution $p$, $q$ defined on an input space $\mathcal{X}$, such that $\forall \x \in \mathcal{X}, q(\x) > 0$, the inequality $\sup_{\x \in \mathcal{X}}[p(\x)/q(\x)] \geq 1$ holds. Moreover, the minimum is reached when $p=q$.
\end{lemma*}
\begin{proof}
    Suppose that $\not\exists \x \in \mathcal{X} s.t. \sup_{\x} p(\x)/q(\x) \geq 1$. This means that $\forall \x, p(\x) < q(\x)$. By integrating those positive and continuous functions on their domains, we are lead to the contradiction that the integral of one of them is not equal to 1. Thus, $\exists \x \in \mathcal{X} s.t. p(\x)/q(\x) \geq 1$. Thus, $\sup_{\x \in \mathcal{X}}[p(\x)/q(\x)] \geq 1$, with equality trivially when $p=q$. 
\end{proof}

\begin{proposition*}[\textbf{\ref{prop:imputation_bound}}]
     Under Assumption \ref{ass:cond_hyp}, given $f \in \mathcal{F}, \gHat$ in \eqref{eq:gHat} and $g$ in \eqref{eq:g}, 
    \begin{equation}
        \begin{aligned}
            \epsilon_T(f\circ g) \leq \underbrace{\underbrace{\sup_{\z\sim p(\Z)}[\dfrac{p_S(\Z_2=\z_2|\z_1)}{p_S(\ZHat_2=\z_2|\z_1)}]}_{\mathrm{Imputation~error~on~ S~} \ref{term_Is}} \times \underbrace{\sup_{\z\sim p(\Z)}[\dfrac{p_S(\ZHat_2=\z_2|\z_1)}{p_T(\ZHat_2=\z_2|\z_1)}]}_{\mathrm{Transfer~error~of~Imputation~} \ref{term_Ti}}}_{\mathrm{Imputation~error~on~T~} \ref{term_It}} \times \epsilon_T(f\circ \gHat)
        \end{aligned}
        \tag{\ref{eq:imputation_bound}}
    \end{equation}
    Under Lemma \ref{lemma:lemma_s_bound}, \ref{term_It}=1 is the minimal value reached when $p_S(\Z_2|\Z_1)=p_S(\ZHat_2|\Z_1)$ and $p_S(\ZHat_2|\Z_1)=p_T(\ZHat_2|\Z_1)$. In this case, $\epsilon_T(f\circ g) = \epsilon_T(f\circ \gHat)$.
\end{proposition*}
\begin{proof}
    We denote $f_T^z$, the latent target labeling function. Moreover, for simplicity, we write $h_{\gHat}=f \circ \gHat$, $h_g=f \circ g$ and $\forall \z \sim p(\Z), S_D(\z)=\dfrac{p_D(\Z_2=\z_2|\z_1)}{p_D(\ZHat_2=\z_2|\z_1)}$
    \begin{align*}
        \epsilon_T(h_{g}) &= \mathbb{E}_{\x_T \sim p_T(X)}[\mathbb{I}(h_{g}(\x_T) \neq f_T(\x_T))] \\
        &= \mathbb{E}_{\ztOne \sim p_T(\Z_1), \ztTwo \sim p_T(\Z_2|\Z_1) }[\mathbb{I}(f(\ztOne, \ztTwo) \neq f_T^z(\ztOne, \ztTwo))] \\
        &= \mathbb{E}_{\ztOne \sim p_T(\Z_1), \ztTwoH \sim p_T(\ZHat_2|\Z_1) }[\dfrac{p_T(\Z_2=\ztTwoH|\ztOne)}{p_T(\ZHat_2=\ztTwoH|\ztOne)} \mathbb{I}(f(\ztOne, \ztTwoH) \neq f_T^z(\ztOne, \ztTwoH))] \\
        &\leq \sup_{\z\sim p(\Z)}[S_T(\z)] \mathbb{E}_{\x_T \sim p_T(X)}[\mathbb{I}(h_{\gHat}(\x_T) \neq f_T(\x_T))] \\
        &= \sup_{\z\sim p(\Z)}[S_T(\z)] \epsilon_T(h_{\gHat})
    \end{align*}
    
    However, $\forall \z \in \mathcal{Z}, S_T(\z)$ cannot be computed as there is not supervision possible on $T$. We will instead apply Assumption \ref{ass:cond_hyp} and use source data for which we can compute $S_S(\z)$.
    \begin{align*}
        \forall \z \in \mathcal{Z}~S_T(\z) &=\dfrac{p_T(\Z_2=\z_2|\z_1)}{p_T(\ZHat_2=\z_2|\z_1)} \\ 
        &=\dfrac{p_S(\Z_2=\z_2|\z_1)}{p_T(\ZHat_2=\z_2|\z_1)} &\text{Assumption~} \ref{ass:cond_hyp}\\
        &=\dfrac{p_S(\Z_2=\z_2|\z_1)}{p_S(\ZHat_2=\z_2|\z_1)} \times \dfrac{p_S(\ZHat_2=\z_2|\z_1)}{p_T(\ZHat_2=\z_2|\z_1)} \\
        &=S_S(\z) \times \dfrac{p_S(\ZHat_2=\z_2|\z_1)}{p_T(\ZHat_2=\z_2|\z_1)} \\
    \end{align*}
    Thus by applying $\sup$,
    \begin{align*}
        \sup_{\z\sim p(\Z)}[S_T(\z)] = \sup_{\z\sim p(\Z)}[S_S(\z)] \times \sup_{\z\sim p(\Z)}[\dfrac{p_S(\ZHat_2=\z_2|\z_1)}{p_T(\ZHat_2=\z_2|\z_1)}]
    \end{align*}
    This yields \eqref{eq:imputation_bound}.
    
    If \ref{term_It}=1 when $p_S(\Z_2|\Z_1)=p_S(\ZHat_2|\Z_1)$ and $p_S(\ZHat_2|\Z_1)=p_T(\ZHat_2|\Z_1)$ per Lemma \ref{lemma:lemma_s_bound}, then $S_T(\z)=1$ and $\epsilon_T(f\circ g) = \epsilon_T(f\circ \gHat)$.
\end{proof}

\begin{proposition*}[\textbf{\ref{prop:self_training}}]
Assume a joint distribution $p_{\widetilde{T}}(X,Y)$ where $p_{\widetilde{T}}(X) = p_T(X)$ and $Y = h_\gHat(X)$ where $h_\gHat = f \circ \gHat \in \mathcal{H}_\gHat$ is a candidate hypothesis. Then,
\begin{equation*}
    \lambda_{\mathcal{H}_{\gHat}} \leq \min_{h_\gHat \in \mathcal{H}_\gHat} \left[\epsilon_S(h_\gHat) + \epsilon_{\widetilde{T}}(h_\gHat) + \epsilon_{T}(f_{\widetilde{T}})\right]
\end{equation*}
with $\epsilon_{T}(f_{\widetilde{T}})=\operatorname{Pr}_{\x \sim p_T(X)}(f_{\widetilde{T}}(\x) \neq f_T(\x))$ the error of the pseudo-labelling function $f_{\widetilde{T}}$ on $T$.
\end{proposition*}
\begin{proof} 
    We know that $p_{\widetilde{T}}(X)=p_T(X)$ as instances are not changed by applying the pseudo-labelling function. Thus, given $h_\gHat \in \mathcal{H}_\gHat$
    \begin{equation*}
        \epsilon_{T}(h_\gHat)=\epsilon_{T}(h_\gHat,f_T)=\epsilon_{\widetilde{T}}(h_\gHat, f_T)
    \end{equation*}
    
    Applying the triangle inequality for classification error \cite{Crammer2006},
    \begin{equation*}
        \epsilon_{\widetilde{T}}(h_\gHat, f_T) \leq \epsilon_{\widetilde{T}}(h_\gHat, f_{\widetilde{T}}) + \epsilon_{\widetilde{T}}(f_{\widetilde{T}}, f_T) 
    \end{equation*}
    
    Finally, we can rewrite $\epsilon_{\widetilde{T}}(h_\gHat, f_{\widetilde{T}}) = \epsilon_{\widetilde{T}}(h_\gHat)$ and $\epsilon_{\widetilde{T}}(f_{\widetilde{T}}, f_T)=\epsilon_{T}(f_{\widetilde{T}}, f_T)=\epsilon_{T}(f_{\widetilde{T}})$.
\end{proof}

\section{Dataset description}
\label{sec:dataset_description}
\subsection{Digits}

We scale all images to $32 \times 32$ and normalize the input in $[-1, 1]$. When adaptation involves a domain with three channels (SVHN or MNIST-M) and a domain with a single channel, we simply triplicate the channel of the latter domain. As in \cite{Damodaran2018} we use balanced source batches which proves to increase performance especially when the source dataset is imbalanced (e.g. SVHN and USPS datasets) while the target dataset (usually MNIST derived) is balanced. Scaling the input images enables us to use the same architecture across datasets. In practise the embedding size is 2048 after preprocessing. For missing versions, we set pixel values to zero in a given patch as shown in Figure \ref{fig:missing_digits}. The \texttt{digits} datasets are provided with a predefined train / test split. We report accuracy results on the target test set and use the source test set as validation set (Section \ref{sec:hyperparameter}). The number of instances in each dataset is reported in Table \ref{table:stats}. We run each model five times.

\begin{table}[h!]
    \centering
    \resizebox{\linewidth}{!}{
        \begin{tabular}{ccccc}
            \hline
            & \multicolumn{1}{c}{$\texttt{USPS}$} & \multicolumn{1}{c}{$\texttt{MNIST}$} & \multicolumn{1}{c}{$\texttt{SVHN}$} & \multicolumn{1}{c}{$\texttt{MNIST-M}$} \\ \hline
            \hline
            Train & 7438 & 60k & 73257 & 60k \\
            Test & 1860 & 10k & 26032 & 10k \\
            Size & 28 $\times$ 28 & 28 $\times$ 28 & 32 $\times$ 32 & 28 $\times$ 28 \\
            Channels & 1 & 1 & 3 & 3 \\
            & \includegraphics[width=0.2\textwidth]{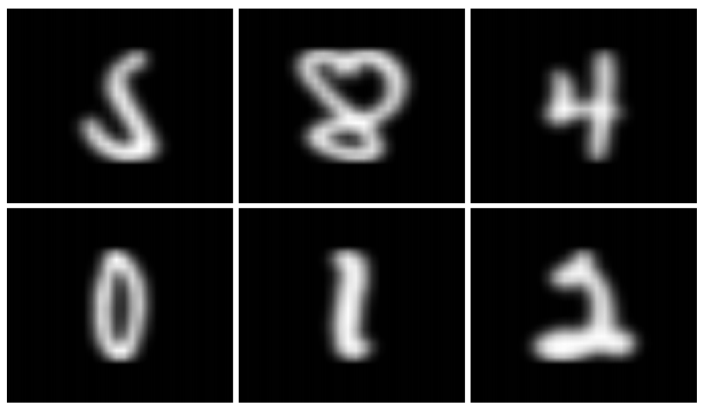} & \includegraphics[width=0.2\textwidth]{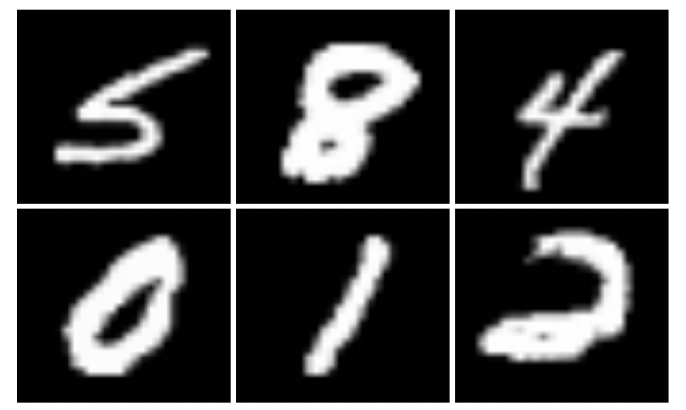} & \includegraphics[width=0.2\textwidth]{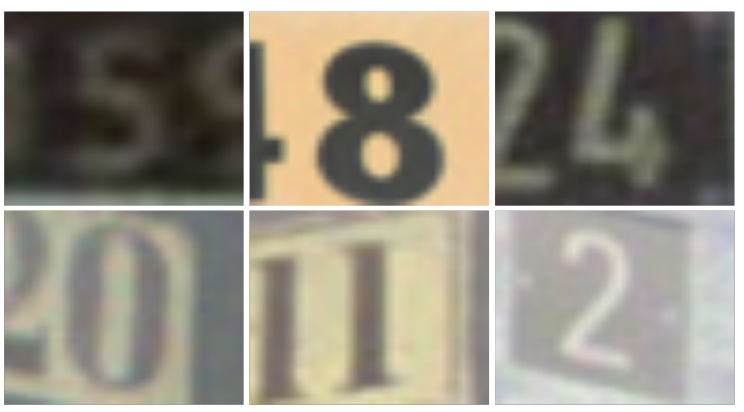} & \includegraphics[width=0.2\textwidth]{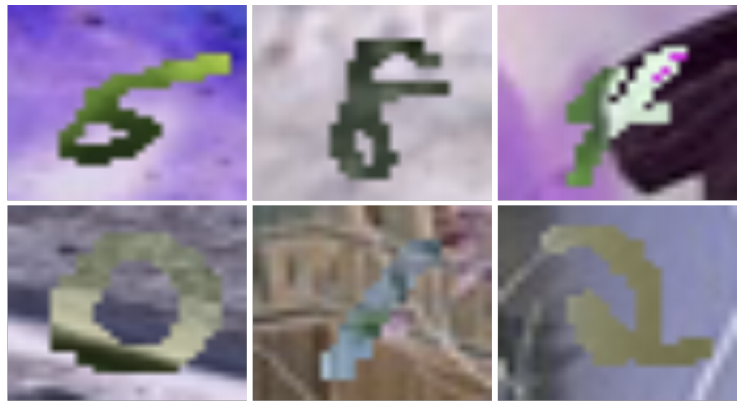} \\
            \hline
        \end{tabular}
    }
    \caption{Statistics on \texttt{digits} datasets}
    \label{table:stats}
\end{table}

\subsection{Amazon}
Each domain has around 2000 samples and we use features freely available at \url{https://github.com/jindongwang/transferlearning/tree/master/data#amazon-review} which follows the data processing pipeline in \cite{Chen2012}. Each review is preprocessed as a feature vector of unigrams and bigrams keeping only the 5000 most frequent features. In practise, we consider the dense version of these features after projection onto a low-dimensional sub-space of dimension 400 with PCA as in \cite{Chen2012}. Datasets with missing features are built by setting the first half of the features to 0. 

\subsection{Ads}

Table \ref{table:stats_criteo} lists statistics on the traffic for the two \texttt{ads} datasets; we now describe how they are preprocessed. On both datasets the train and test sets are fixed. We run each model five times.
\begin{table}[ht!]
    \centering
    \resizebox{0.9\linewidth}{!}{
        \begin{tabular}{ccccccccc}
            \hline
            Dataset & \multicolumn{4}{c}{\texttt{ads-kaggle}} & \multicolumn{4}{c}{\texttt{ads-real}}  \\ \hline
            \hline
            Domain & \multicolumn{2}{c}{Source} & \multicolumn{2}{c}{Target} & \multicolumn{2}{c}{Source} & \multicolumn{2}{c}{Target} \\ \hline
            \hline
            Split & \texttt{Train} & \texttt{Test} & \texttt{Train} & \texttt{Test} & \texttt{Train} & \texttt{Test} & \texttt{Train} & \texttt{Test} \\ \hline
            \hline
            Positive & 246.872 & 61.841 & 92.333 & 22.943 & X & X & X & X \\
            Negative & 699.621 & 174.783 & 854.160 & 213.681 & X & X & X & X \\
            Total & 946.493 & 236.624 & 946.493 & 236.624 & 24.465.756 & 3.760.233 & 819.073 & 147.358 \\
            $p(Y=1)$ & 0,2608 & 0,2613 & 0,0976 & 0,0970 & X & X & X & X \\
            \hline
        \end{tabular}
    }
    \caption{Statistics on \texttt{ads} datasets}
    \label{table:stats_criteo}
\end{table}

\paragraph{$\texttt{ads-kaggle}$} 
The Criteo Kaggle dataset is a reference dataset for CTR prediction and groups one week of log data. The objective is to model the probability that a user will click on a given ad given his browsing behaviour. Positives refer to clicked ads and negatives to non-clicked ads. For each datum, there are 13 continuous and 26 categorical features. We divide the traffic into two domains using feature number 30 corresponding to an engagement feature; for a given value for this categorical feature, all instances have a single missing numeric feature (feature number 5). We then construct an artificial dataset simulating transfer between known and new users. We process the original Criteo Kaggle dataset to have an equal number of source and target data. We then perform train / test split on this dataset keeping 20$\%$ of data for testing. We used in our experiments only continuous features; to show the benefit of modelling additional missing features, we extend the missing features list to features 1, 5, 6, 7, 11 and 12 by setting them to zero on the target domain. After these operations, 6 features are missing and 7 are non-missing. Preprocessing consists in normalizing continuous features using a log transform.

\paragraph{$\texttt{ads-real}$} 
This private dataset is similar to $\texttt{ads-kaggle}$. We filter out non-clicks and the final task is to model the sale probability for a clicked ad given an user's browsing history. Positives refer to clicked ads which lead to a sale; negatives to clicked ads which did not lead to a sale. We use one week of sampled logs as a training set and use the following day data as the test set. This train / test definition is used so to better correlate with the performance of a production model. Features are aggregated across user timelines and describe the clicking and purchase behavior of a user. In comparison to $\texttt{ads-kaggle}$ more continuous features are used. The count features can be User-centric i.e. describe the global activity of the user (number of clicks, displays, sales done globally across partners) or User-partner features i.e. describing the history of the user on the given partner (number of clicks, sales... on the partner). The latter are missing for prospecting users. Counts are aggregated across varying windows of time and categories of partner catalog. We bucketize these count features using log transforms and project the features into an embedding space of size $596$ with 29 features. 12 features are missing and 17 are non-missing. 

\newpage
\begin{minipage}{\linewidth}
    \centering
    \scalebox{0.7}{
        \begin{tabular}{ccccc}
            \hline
            Domain & \multicolumn{1}{c}{Source} & \multicolumn{1}{c}{Target} \\ \hline
            feature 1 & 0.80$\rpm$2.21 & 4.4 $\times 10^{-4}\rpm$0.041 \\
            feature 2 & 9.16$\rpm$13.04  & 9.01$\rpm$13.42 \\
            feature 3 & 4.40$\rpm$6.32 & 3.44$\rpm$6.19 \\
            feature 4 & 2.58$\rpm$3.27  & 0.94$\rpm$2.31 \\
            feature 5 & 61.09$\rpm$37.67 & 0.0$\rpm$0.0 \\
            feature 6 & 11.26$\rpm$12.24 & 0.090$\rpm$1.69 \\
            feature 7 & 4.10$\rpm$6.23 & 0.0034$\rpm$0.13\\
            feature 8 & 5.12$\rpm$4.50 & 1.91$\rpm$4.26 \\
            feature 9 & 14.32$\rpm$11.57 & 3.273$\rpm$5.36\\
            feature 10 & 0.046$\rpm$0.22 & 1.35 $\times 10^{-5} \rpm$0.0037 \\
            feature 11 & 1.08$\rpm$2.11 & 4.25 $\times 10^{-4}\rpm$0.029 \\
            feature 12 & 0.083$\rpm$0.78 & 6.68 $\times 10^{-5}\rpm$0.018 \\
            feature 13 & 2.74$\rpm$3.59 & 1.21$\rpm$3.36 \\
            \hline
        \end{tabular}
    }
    \captionof{table}{Feature mean and standard deviation on $\texttt{ads-kaggle}$. We set features 1, 6, 7, 11, 12 to zero on $T$.}
    \label{table:hist_tab_kaggle}
    
    \centering
    \includegraphics[width=0.8\textwidth]{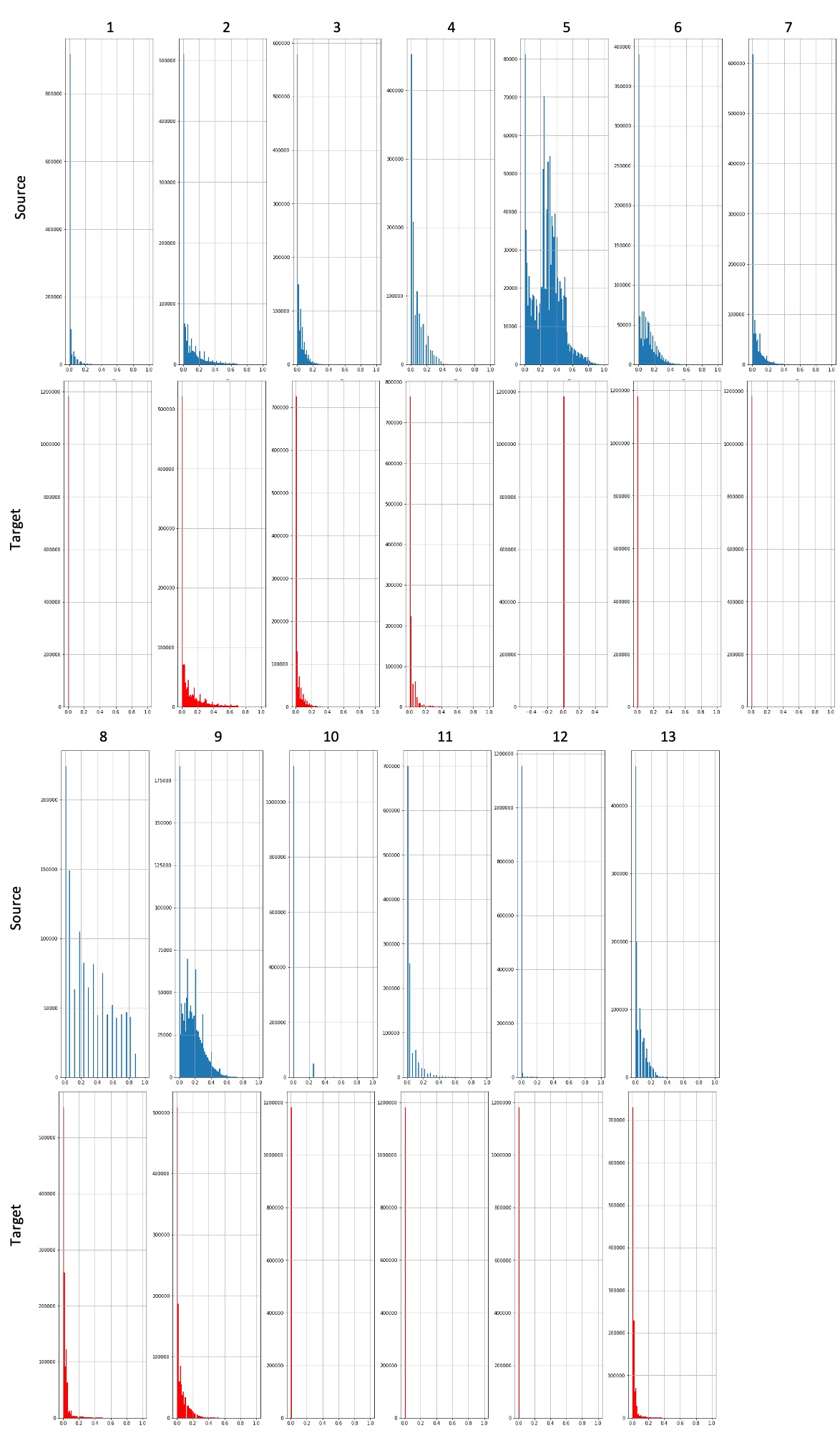}
    \captionof{figure}{Source (blue) and Target (red) distributions on \texttt{ads-kaggle} for each feature (1 to 13)}
    \label{fig:hist_kaggle}
\end{minipage} 

\section{Implementation details}
\label{sec:implementation_details}

\subsection{Neural Net architecture}

\paragraph{digits}
We use the \texttt{ADV} and \texttt{OT} versions of our imputation model. For \texttt{ADV} models, we use the DANN model description in \cite{Ganin2015}; for \texttt{OT} we use the DeepJDOT model description in \cite{Damodaran2018}. Both models can be considered as simplified instances of our corresponding \texttt{ADV} and \texttt{OT} imputation models when no imputation is performed. Performance of the adaptation models is highly dependent on the NN architectures used for adaptation and classification. In order to perform fair comparisons and since our goal is to evaluate the potential of joint Adaptation-imputation-classification, we selected these architectures through preliminary tests and use them for both the \texttt{ADV} and \texttt{OT} models. The two models are described below and illustrated in Figure \ref{fig:adv_archi}.

\begin{itemize}[leftmargin=0.5cm]
    \item Feature extractors $g_1$ and $g_2$ consists of three convolutional layers with $5 \times 5$ kernel and 64 filters interleaved with max pooling layers with a stride of 2 and $2 \times 2$ kernel. The final layer has 128 filters. We use batch norm on convolutional layers and ReLU as an activation after max pooling. As in \cite{Damodaran2018} we find that adding a sigmoid activation layer as final activation is helpful.
    \item Classifier $f$ consists of two fully connected layers with 100 neurons with batch norm and ReLU activation followed by the final softmax layer. We add Dropout as an activation for the first layer of the classifier.
    \item Discriminator $D_1$ and $D_2$ is a single layer NN with 100 neurons, batch norm and ReLU followed by the final softmax layer. On USPS $\rightarrow$ MNIST and MNIST $\rightarrow$ USPS dataset we use a stronger discriminator network which consists of two fully connected layers with 512 neurons.
    \item Generator $r$ consists of two fully connected layers with 512 neurons, batch norm and ReLU activation. This architecture is used for \texttt{ADV} and \texttt{OT} imputation models. In practice using wider and deeper networks increases classification performance with the more complicated classification tasks (SVHN $\rightarrow$ MNIST, MNIST $\rightarrow$ MNIST-M); in these cases we add an additional fully connected network with 512 neurons. The final activation function is a sigmoid.
\end{itemize}
We use the same architecture described above for all our models to guarantee fair comparison. As a side note, the input to the imputation model's classifier is twice bigger as in the standard adaptation models.

\begin{figure}[h!]
\centering
\includegraphics[width=0.9\textwidth]{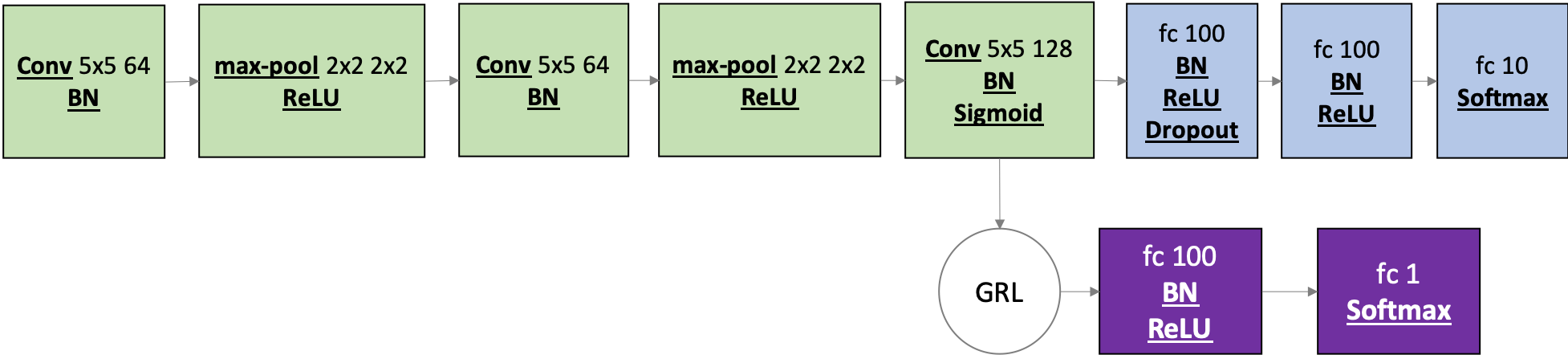}
\caption{Base architecture for the \texttt{ADV} DANN model}
\label{fig:adv_archi}
\end{figure}

\paragraph{\texttt{ads-kaggle} and \texttt{amazon}}
We experiment with \texttt{ADV} models only. As input data is numeric and low dimensional, architectures are simpler than in \texttt{digits}. Our feature extractor is a three layered NN with 128 neurons and with a final sigmoid activation. The classifier is taken to be a single layered NN with 128 neurons and a final softmax layer. Activations are taken to be ReLUs. The domain discriminator is taken to be a two layered NN with 128 neurons and a final softmax layer. Finally the reconstructor is taken to be a two-layered NN with 256 neurons and final sigmoid activation.

\paragraph{\texttt{ads-real}}
We experiment with \texttt{ADV} models only. Input features after processing are fed directly into the feature extractors $g_1$, $g_2$ consisting of two fully connected layers with 128 neurons. The classifier and discriminator is taken to be single-layered NN with 25 neurons. The reconstructor is taken to be a two-layered NN with 128 neurons. Inner activations are taken to be ReLUs and the final activation of the feature extractor is taken to be a sigmoid.

\subsection{Network parameters}

\subsubsection{Hyperparameter tuning}
\label{sec:hyperparameter}

Tuning hyperparameters for UDA is tricky as we do not have access to target labels and thus cannot choose parameters minimizing the target risk on a validation set. Several papers set hyperparameters through reverse cross-validation \cite{Ganin2015}. Other approaches developed for model selection are based on risk surrogates obtained by estimating an approximation of the risk value on the source based on the similarity of source and target distributions (without the labels). In the experiments, we used a recent estimator, Deep Embedded Validation (DEV) \cite{You2019} for tuning the initial learning rate and for the \texttt{OT} imputation model, tuning $\lambda_{1}$ and $\lambda_{OT}$. For other parameters, we used heuristics and typical hyperparameter values from UDA papers (such as batch size) without further tuning.
We use a cross entropy link function on the source validation set; this value provides a proxy for the target test risk. Using parameters from the original paper, this estimator helps select parameter ranges which perform reasonably well. We keep the estimator unchanged for our baseline models. In the imputation case, the discriminator used for computing importance sampling weights discriminates between $\zsH$ and $\ztH$ i.e. $D_1$ (Figure \ref{fig:imput_model}).

\subsubsection{Digits}

We find that the results are highly dependent on the NN architecture and the training parameter setting. In order to evaluate the gain obtained with \textit{Adaptation-Imputation}, we use the same NN architecture for all models (\texttt{ADV} and \texttt{OT}) but fine tune the learning rates for each model using the DEV estimator (other parameters do not have a significant impact on the classification performance).

\paragraph{\texttt{ADV}} We use an adaptive approach as in \cite{Ganin2015} for decaying the learning rate $lr$ and updating the gradient's scale $s$ between 0 and 1 for the domain discriminators.
We choose the decay values used in \cite{Ganin2015} ie. $s = \dfrac{2}{1 + \exp(-10 \times p)} - 1$ and $lr = \dfrac{lr_i}{(1 + 10 \times p)^{0,75}}$ where $p$ is ratio of current batches processed over the total number of batches to be processed without further tuning. We tune the initial learning rate $lr_i$, chosen in the range $\{10^{-2}, 10^{-2.5}, 10^{-3}, 10^{-3.5}, 10^{-4}\}$ following Section \ref{sec:hyperparameter}. In practise we take $lr_i = 10^{-2}$ for \texttt{ADV} \textit{Adaptation-Imputation}, \textit{Adaptation-Full}, \textit{Adaptation-IgnoreComponent} and $lr_i = 10^{-2.5}$ for \texttt{ADV} \textit{Adaptation-ZeroImputation}. We use Adam as the optimizer with momentum parameters $\beta_1=0.8$ and $\beta_2=0.999$ and use the same decay strategy and initial learning rate for all components (feature extractor, classifier, reconstructor). Batch size is chosen to be 128; we see in practise that initializing the adaptation models with a source model with smaller batch size (such as 32) can be beneficial.

\paragraph{\texttt{OT}} We choose parameter $\lambda_{OT} = 0.1$ in Equation \ref{eq:imput_ot} after tuning in the range $\{10^{-1}, 10^{-2}, \\ 10^{-3}\}$ using DEV. We weight $L_1$ in Equation \ref{eq:overall_loss} by $\lambda_{1} = 0.1$. Following \cite{Damodaran2018}, batch size is taken to be 500 and we use EMD a.k.a. Wasserstein-2 distance. We initialize adaptation models with a source model in the first 10 epochs and divide the initial learning rate by two as adaptation starts for non-imputation models. For \textit{Adaptation-Imputation} we follow a decaying strategy on the learning rate and on the adaptation weight as explained in the next item. We choose $lr_i$ in the range $\{10^{-2}, 10^{-2.5}, 10^{-3}, 10^{-3.5}, 10^{-4}\}$. In practise we fix $lr_i = 10^{-2}$ for all models.

\paragraph{Imputation parameters} Ablation studies are conducted in Section \ref{sec:ablation_studies} on weights in Equation \ref{eq:imput_l2}; in \texttt{digits} experiments we choose $L_2 = L_{MSE} + L_{ADV}$ for \texttt{ADV} and \texttt{OT} to reduce the burden of additional feature tuning. For \texttt{ADV} model, we fix $\lambda_1=\lambda_2=\lambda_3=1$ in Equation \ref{eq:overall_loss}. In the \texttt{OT} model, we vary $\lambda_1$ between 0 and 0.1 and $\lambda_2$ between 0 and 1 following the same schedule as the gradient scale update for \texttt{ADV} models to reduce variance. 

\subsubsection{Ads}

We use an adaptive strategy for updating the gradient scale and the learning rate with the same parameters as in the $\texttt{digits}$ dataset. Optimizer is taken to be Adam. Batch size is taken to be big so that target batches include sufficient positive instances. 

\paragraph{\texttt{ads-kaggle}}
The initial learning rate is chosen in the range $\{10^{-4}, 10^{-5}, 10^{-6}, 10^{-7}\}$ using DEV and fixed to be $10^{-6}$ for all models. Batch size is taken to be 500 and we initialize models with a simple classification loss for five epochs. We run models for 50 epochs after which we notice that models reach a plateau. We find that adding a weighted MSE term allows to achieve higher stability (as measured by variance) as further studied in Section \ref{sec:ablation_studies}. In a similar fashion to \cite{Pathak2016}, we tune this weight in the range $\{1, 10^{-1}, 10^{-2}, 7.5 \times 10^{-3}, 5 \times 10^{-3}, 10^{-3}\}$. We find that $0.005$ offers the best compromise between mean loss and variance. Moreover on this dataset we use a faster decaying strategy for the discriminator's $D_2$ and the reconstructor's $r$ learning rate, $lr = \dfrac{lr_i}{(1 + 30 \times p)^{0,75}}$ to achieve higher stability in the training curves while the feature extractor $g_1$, $g_2$ and $D_1$'s learning rate are unchanged.

\paragraph{\texttt{ads-real}}
The initial learning rate is chosen in the range $\{10^{-4}, 10^{-5}, 10^{-6}\}$ and fixed to be $10^{-6}$ for all models. The learning rate is decayed with the same parameters as \texttt{digits} for all models. We run models for ten epochs which provides a good trade-off between learning time and classification performance. Batch size is taken to be 500. We choose $L_2 = L_{MSE} + L_{ADV}$ without further tuning; this achieves already good results. 

\subsection{Amazon}
We use the same hyperparameters as \texttt{ads-kaggle}. $\lambda_{MSE}$ is set to 1 without further tuning.

\section{Latent space visualization on \texttt{digits}}
\label{sec:appendix_embedding}

In this section we visualize the embeddings $\zH=\gHat(\z)$ learned by the various models on \texttt{digits} by projecting the embeddings in a 2D space using $\gHat$ with t-SNE (the original embedding size being 2048). Figure \ref{fig:embedding_adv_m_m} represents the embeddings learned for \texttt{ADV} models on MNIST $\rightarrow$ MNIST-M. Figures \ref{fig:embedding_ot_m_m} and \ref{fig:embedding_ot_m_u} represent these embeddings for \texttt{OT} models respectively on MNIST $\rightarrow$ MNIST-M and MNIST $\rightarrow$ USPS. On these figures, we see that \textit{Adaptation-Imputation} generates feature representations that overlap better between source and target examples per class than the adaptation counterparts (although \textit{Adaptation-IgnoreComponent} does a good job at overlapping feature representations). This correlates with the accuracy performance on the target test set. Moreover we notice, as expected, that \textit{Adaptation-IgnoreComponent} and \textit{Adaptation-ZeroImputation} perform badly compared to \textit{Adaptation-Full} which justifies the use of \textit{Adaptation-Imputation} when confronted to missing non-stochastic data.

\begin{figure}[ht!]
\centering
\includegraphics[width=0.71\textwidth]{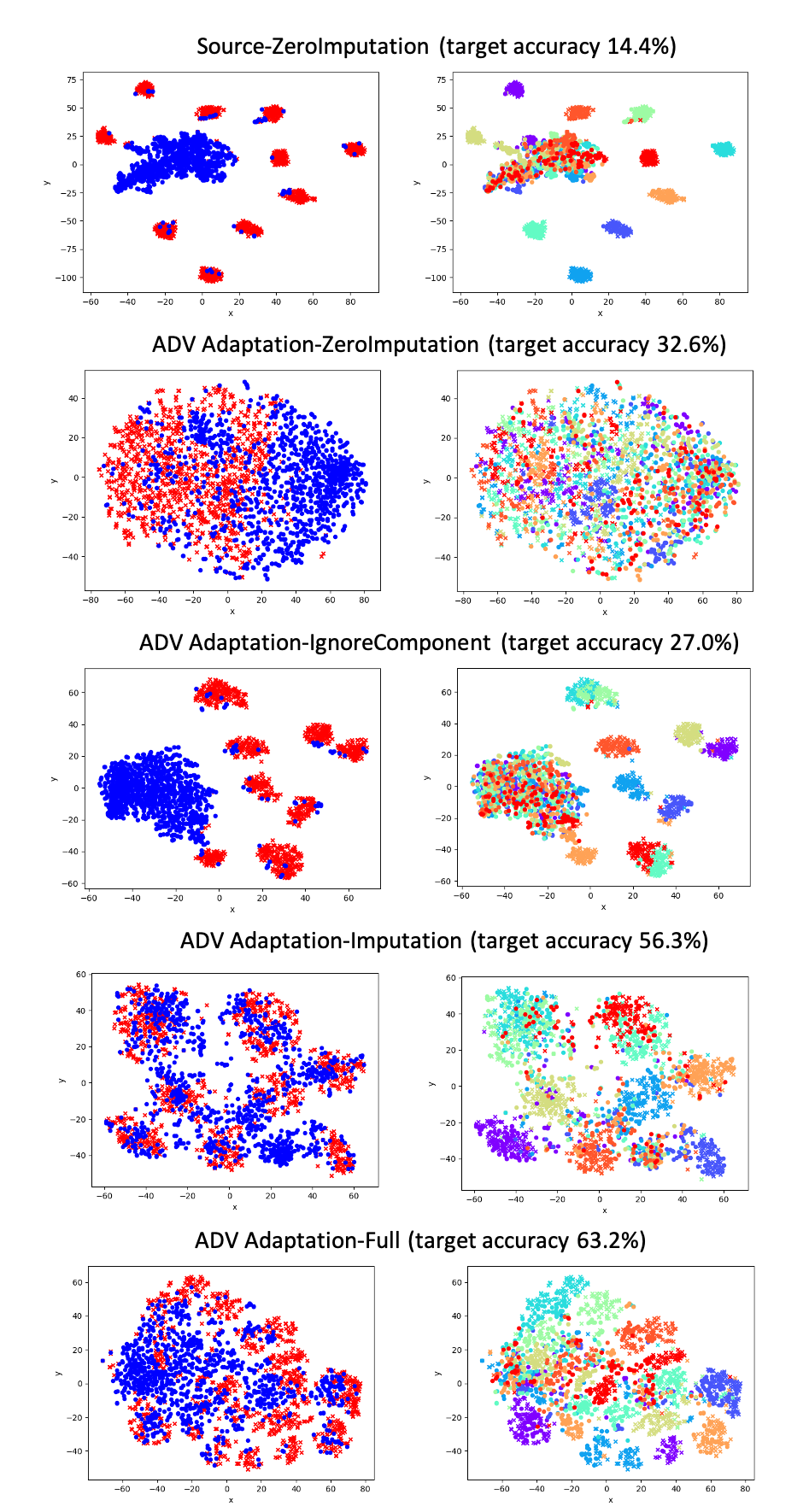}
\caption{Embeddings for MNIST $\rightarrow$ MNIST-M dataset for \texttt{ADV} models on a batch. Figures on the left represent the source (red) and target (blue) clusters; Figures on the right represent the classes on source and target.}
\label{fig:embedding_adv_m_m}
\end{figure}

\begin{figure}[ht!]
\centering
\includegraphics[width=0.71\textwidth]{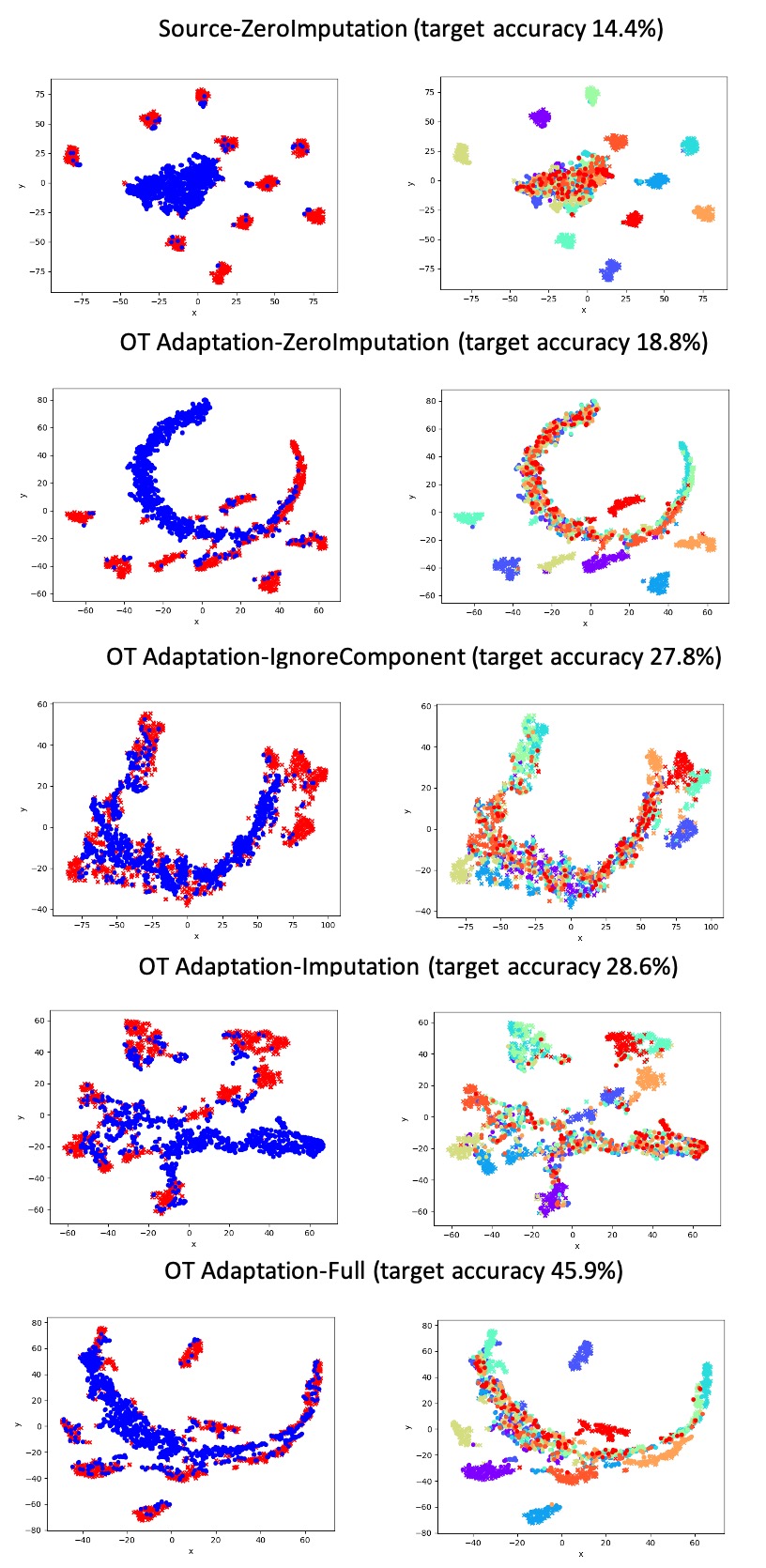}
\caption{Embeddings for MNIST $\rightarrow$ MNIST-M dataset for \texttt{OT} models on a batch. Figures on the left represent the source (red) and target (blue) clusters; Figures on the right represent the classes on source and target.}
\label{fig:embedding_ot_m_m}
\end{figure}

\begin{figure}[ht!]
\centering
\includegraphics[width=0.71\textwidth]{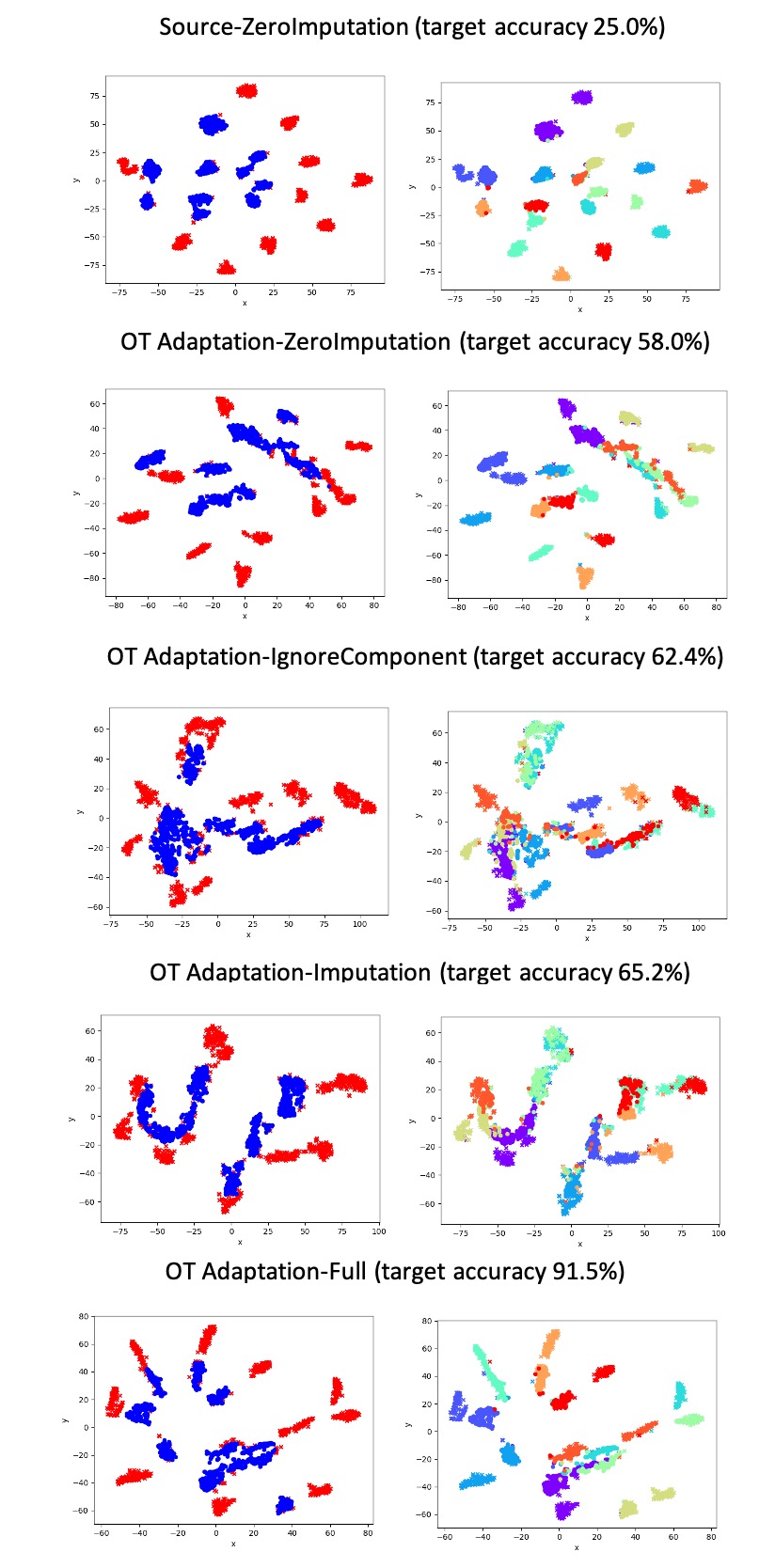}
\caption{Embeddings for MNIST $\rightarrow$ USPS dataset for \texttt{OT} models on a batch. Figures on the left represent the source (red) and target (blue) clusters; Figures on the right represent the classes on source and target.}
\label{fig:embedding_ot_m_u}
\end{figure}
\FloatBarrier

\end{document}